\documentclass[5p,twocolumn]{elsarticle}
\usepackage{amsmath,amssymb,amsthm}
\usepackage{booktabs}
\usepackage{xcolor}
\usepackage{algorithm}
\usepackage{algpseudocode}

\newtheorem{theorem}{Theorem}
\newtheorem{remark}{Remark}

\newcommand{\R}{\mathbb{R}}

\journal{Computer Methods in Applied Mechanics and Engineering}

\begin{document}

\begin{frontmatter}

\title{Convex Neural Energy Elements: Monolithic Finite-Element Assembly of
Geometry-Parameterized Neural Operators with Stability and Error
Guarantees}

\author[a,b]{Hongyue Jiang}
\author[c,b]{Jianjiang Zhan}
\author[a,b]{Chenzhuo Zhang}
\author[a,b]{Fan Wang\corref{cor1}}
\cortext[cor1]{Corresponding author.}
\ead{wang\_fan@alumni.hust.edu.cn}

\affiliation[a]{organization={School of Civil and Hydraulic Engineering,
Huazhong University of Science and Technology},
city={Wuhan}, postcode={430074}, country={China}}
\affiliation[b]{organization={National Center of Technology Innovation for
Digital Construction, Huazhong University of Science and Technology},
city={Wuhan}, postcode={430074}, country={China}}
\affiliation[c]{organization={School of Artificial Intelligence and Automation,
Huazhong University of Science and Technology},
city={Wuhan}, postcode={430074}, country={China}}

\begin{abstract}
Extending the neural-operator element method from individually trained,
fixed-geometry neural elements to a \emph{library} of reusable,
geometry-parameterized element types fails structurally: a field-predicting
operator trained by value
regression induces an energy whose assembled Hessian is indefinite, and
Newton converges to spurious minima (247\% error) even with 1\%-accurate
field predictions. We introduce \emph{convex neural energy elements}: each
element exports a scalar energy $E(g,U)$, architecturally convex in its
boundary degrees of freedom $U$ and smoothly parameterized by its geometry
$g$, realized as a hypernetwork-generated positive-semidefinite quadratic
form (an input-convex correction is reserved for non-quadratic physics).
A regularization-nullspace principle---the regularizer's nullspace must
contain the physics nullspace---removes an otherwise irreducible bias, and
assembled elements inherit the classical guarantee that singular element
stiffnesses yield a positive-definite global system. We prove conditional
error bounds (energy-to-solution accuracy, element-count scaling, geometry
generalization) and verify each experimentally. On heat conduction with
elliptic holes, one trained element assembles into $2\times2$ to
$8\times8$ grids and an L-shaped layout of unseen geometries at
$0.6$--$1.0\%$ relative $L^2$ error, with $175\times$ faster per-geometry
setup for boundary-quantity workloads. A second trained element type
mixes freely with the first in one monolithic assembly, and a
three-dimensional instantiation reaches $0.23\%$ on eight-element
assemblies---the guarantees are type- and dimension-agnostic. A
plane-strain elasticity element, whose physics nullspace is
three-dimensional, lands on the analytically predicted regularization
floors. Making the energy the learned object
turns neural operators from single-use surrogates into reusable elements
that inherit the assembly guarantees of the method they extend.
\end{abstract}


\end{frontmatter}

\section{Introduction}

Operator learning~\citep{lu2021deeponet,li2021fno,kovachki2023neural}
promises reusable surrogates for the expensive inner loops
of computational mechanics. Two recent lines of work frame the opportunity.
The neural-operator element method (NOEM)~\citep{ouyang2026noem} embeds a
pre-trained neural operator (NO) as a single macro-element inside a standard
finite-element (FE) solve, so that a subdomain that would require thousands of
elements is represented by one neural-operator element (NOE). DIMON
\citep{yin2024dimon} learns solution operators over a diffeomorphic family of
domains, demonstrating geometry generalization from a reference domain.
Combining the two suggests an appealing target: a \emph{neural isoparametric
element library} (NIEL)---reference cell, learned shape behavior, and explicit
geometry parameters, mirroring the classical isoparametric
element~\citep{hughes2000fem,zienkiewicz2013fem}---in which
one trained element type is instantiated many times with different geometric
parameters and assembled monolithically, so that a new problem requires
choosing elements and writing a mesh, not retraining networks.

This paper shows that the direct route to that target fails for a structural
reason, identifies the mathematical property that repairs it, and delivers the
repaired construction with theory and experiments.

\paragraph{The failure.} An NOE trained by \emph{value regression}---match the
predicted field to FEM solutions---can be an excellent field predictor and
still be unusable as an assembly building block. Assembly solvers use the
element's internal force $q=\partial E/\partial U$ and tangent stiffness
$K=\partial^2 E/\partial U^2$, i.e., derivatives of the induced energy, and
value training constrains neither the location of the energy's minimum nor the
signature of its Hessian. In our benchmark (Section~\ref{sec:failure}), a
value-trained element predicts interior fields to $0.7$--$1.3\%$ yet its
assembled tangent has 36 negative eigenvalues \emph{at the true solution};
Newton "converges" (step norms below $10^{-6}$) to a spurious critical point
with 247\% solution error. Accuracy of the field predictor and correctness of
the induced variational structure are separate properties; only the first is
controlled by standard training.

\begin{figure*}[t]
\centering
\includegraphics[width=\textwidth]{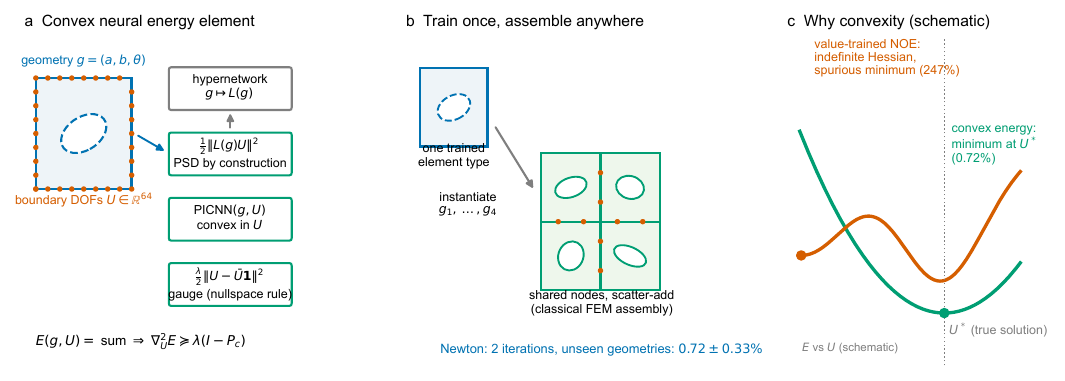}
\caption{Convex neural energy elements at a glance.
(a)~Each element carries a scalar energy $E(g,U)$ of the element's boundary
degrees of freedom $U$, smoothly parameterized by its geometry $g$: a
hypernetwork-generated quadratic head (positive semidefinite by
construction), a partially input-convex correction, and a gauge Tikhonov
regularizer whose nullspace contains the physics nullspace. The sum is
architecturally convex in $U$ (Theorem~\ref{thm:psd}).
(b)~One trained element type is instantiated at different geometries and
assembled monolithically by classical scatter-add over shared boundary
nodes; on unseen geometries the assembly reaches $0.72\pm0.33\%$ relative
$L^2$ error with two Newton iterations (Section~\ref{sec:experiments}).
(c)~The reason convexity is the admissibility condition: a value-trained
field-predicting element induces an energy with an indefinite Hessian and a
spurious minimum to which Newton converges (247\% error), whereas the
convex energy has its minimum at the true solution
(Section~\ref{sec:failure}). Panels (a) and (c) are schematic; all numbers
are measured.}
\label{fig:overview}
\end{figure*}

\paragraph{The repair.} We make the energy itself the learned object and make
its convexity architectural (Figure~\ref{fig:overview}). A \emph{convex
neural energy element} carries
\begin{equation}
E(g, U) \;=\; \tfrac12 \|L(g)\,U\|^2 \;+\;
\mathrm{PICNN}(g, U) \;+\; \tfrac{\lambda}{2}\,\|U - \bar U \mathbf{1}\|^2,
\label{eq:element}
\end{equation}
where $U\in\R^{n}$ are the element's boundary degrees of freedom, $g$ its
geometry parameters, $L(g)$ a hypernetwork-generated matrix (so the quadratic
part $L^\top L\succeq 0$ for every $g$), $\mathrm{PICNN}$ a partially
input-convex network~\citep{amos2017icnn} (convex in $U$, arbitrary in $g$)
that carries non-quadratic physics, and the last term a \emph{gauge} Tikhonov
regularizer acting only on the deviatoric part of $U$
($\bar U$ is the mean). The demonstrated, quantitatively validated scope
of this paper is the quadratic instantiation ($\mathrm{PICNN}$ off): all
linear self-adjoint benchmarks deploy the first and third terms alone.
The correction term is exercised in the nonlinear extension
(Section~\ref{sec:nonlinear}), where the ansatz is proved lossless; a
reduced-basis convex network already learns the correction well enough to
outperform the no-learning physics base threefold, while reaching the
accuracy of the training-free physics ladder remains an identified open
problem. Two structural choices matter:

\emph{(i) Nullspace principle.} The tile's condensed physical energy is
invariant under constant shifts (pure-Neumann condensation), so its Hessian
has a nullvector $\mathbf{1}$. A full Tikhonov term $\lambda\|U\|^2/2$ forces
the model's curvature to exceed $\lambda$ in that direction, which no convex
architecture can cancel; we show analytically that this imposes an
irreducible energy-error floor proportional to $\lambda$, and measure the
model sitting exactly on that floor across three decades of $\lambda$
(Section~\ref{sec:ablations}). The gauge regularizer's nullspace contains the
physics nullspace and removes the floor entirely. The principle---
\emph{the regularizer nullspace must contain the physics nullspace}---is
general, and we exercise it in three regimes: the scalar benchmark (1-dim
nullspace, constants), a plane-strain elasticity instantiation (3-dim
nullspace, rigid modes, where a regularizer missing \emph{only the
rotation} already incurs a measurable floor; Section~\ref{sec:elast}), and
the nonlinear benchmark (empty nullspace: the reaction term pins
constants, and an ordinary small Tikhonov term is again admissible).

\emph{(ii) Classical assembly semantics.} With the gauge regularizer each
element stiffness is positive semidefinite with (at most) the admitted
physics nullspace---constants here, rigid modes in elasticity---exactly
like a classical element stiffness with a rigid mode.
Positive definiteness of the assembled system is then supplied by
connectivity and Dirichlet data, the same mechanism as in classical FEM
(Theorem~\ref{thm:global}); neural elements inherit the assembly theory of
the method they extend.

\paragraph{Contributions.}
\begin{enumerate}
\item \textbf{Diagnosis.} A controlled demonstration that value-trained
neural operators fail as assembly elements for structural reasons (indefinite
induced Hessians, spurious Newton fixed points), isolated from all other error
sources by an exact-condensation control with 0.07\% assembly error.
\item \textbf{Construction.} Convex neural energy elements
\eqref{eq:element}: geometry-parameterized, architecturally convex,
regularized by the nullspace principle, assembled monolithically with
classical FEM semantics; plus an analytic-Hessian fast path for the quadratic
case.
\item \textbf{Theory.} Quantitative bounds: energy accuracy $\epsilon$
controls solution error ($\|U^*_{NN}-U^*\|^2 \le 8\epsilon/\mu$), tile-count
scaling at most linear in $N$, architectural positive-semidefiniteness, and
geometry-generalization with linear dependence on the distance to the
training family---each stated with measurable constants and verified
point-by-point in experiments (all nine model-quality points within the
Theorem~\ref{thm:t1} bound; observed tile-scaling exponent 0.13 vs.\ bound
1, measured up to 64 elements; observed geometry exponent 0.9, matching the
linear-in-$d$ bound). The
nullspace principle and the assembly theorem are additionally exercised on
a second physics---plane-strain elasticity---whose nullspace is
three-dimensional (Section~\ref{sec:elast}).
\item \textbf{Toward an element library.} One trained element type covers a
3-parameter geometry family: assemblies of $2$--$64$ instances at unseen
geometries---rectangular grids and an L-shaped topology with a re-entrant
corner---reach $0.6$--$1.0\%$ mean error ($0.72\pm0.33\%$ over 50
four-element runs, $0.40\pm0.09\%$ on the L-shape), Newton amounting to
one sparse solve plus a verification step;
per-new-geometry setup is $215\,\mu$s vs.\ 38\,ms for mesh-and-condense; the
element outperforms \emph{local} operator interpolation
\citep{amsallem2011online} with identical anchors by $5$--$8\times$ (in
both the stiffness and the square-root variable) and a whole-domain
surrogate with $4\times$ more training data by $11\times$, while a global
polynomial fit of the same square-root target matches it---evidence that
the guarantees come from the target and its convex reconstruction, not
from the network (Section~\ref{sec:ablations}). A second element type (a
two-hole topology with its own geometry space and hypernetwork, trained by
the same recipe) mixes freely with the first in one monolithic assembly at
$0.46\pm0.11\%$---the assembly guarantee is type-agnostic, and adding a
type costs one training run (Section~\ref{sec:hetero}). A 3D
instantiation (cube elements, ellipsoidal voids, 386 interface nodes)
assembles eight unseen geometries at $0.23\pm0.06\%$ with the identical
recipe (Section~\ref{sec:3d}).
Boundary-condition robustness is structural: far
out-of-distribution boundary data (ramps, plateaus) incur no measurable
degradation, in contrast to field-predicting NOEs.
\item \textbf{Beyond quadratic energies.} For a nonlinear
reaction--diffusion benchmark whose condensed element energy is convex but
non-quadratic: (i) assembling \emph{exactly} condensed nonlinear energies
reproduces the monolithic nonlinear solve to machine precision---the
assembly framework transfers without modification; (ii) the
interior-mediated energy correction is provably convex (Schur-complement
monotonicity), so the convex-element ansatz loses nothing in principle;
(iii) a training-free ladder of physics-structured convex trial energies,
made variationally consistent by differentiating through the interior
Newton refinement, reaches $0.040\%$ at two refinement steps and
condensation-level accuracy at three, at Newton-iteration parity with
exact condensation (ignoring the nonlinearity costs 25.7\%); and (iv) a
reduced-basis convex-network correction, trained gradient-first on the
principal curvature directions of the remainder, is the first learned
model to outperform the no-learning physics base ($4.9\%$ vs.\
$14.6\%$)---though it still trails the physics ladder. We identify the
near-singular anisotropic curvature of the correction as the remaining
obstacle and leave fully learned, geometry-parameterized nonlinear
corrections as a precisely-posed open problem.
\end{enumerate}

\section{Background and failure diagnosis}\label{sec:failure}

\subsection{Setting and exact element energies by static condensation}
We consider self-adjoint problems whose discrete solution minimizes an
energy: find $U^*=\arg\min_U \Pi(U)$ with
$\Pi(U)=\sum_{e} E^{(e)}(U|_e) - F^\top U$, the classical variational
structure that FEM assembly and Newton solvers presuppose
\citep{hughes2000fem}. The running linear example is heat conduction,
$-\Delta T=0$ on a domain assembled from square tiles of size $0.8$, each
perforated by an elliptic hole with geometry $g=(a,b,\theta)$ (semi-axes and
rotation), insulated hole boundary, and Dirichlet data on the outer
boundary. Table~\ref{tab:notation} collects the notation used throughout.

\begin{table}[t]
\centering\footnotesize
\caption{Notation.}
\label{tab:notation}
\setlength{\tabcolsep}{3pt}
\begin{tabular}{ll}
\toprule
Symbol & Meaning \\
\midrule
$g\in\mathcal G$ & element geometry parameters ($(a,b,\theta)$ here) \\
$U\in\R^n$ & element boundary degrees of freedom \\
$S_g$ & exact condensed (Schur) stiffness of tile $g$ \\
$E_{\mathrm{FEM}},E_\theta$ & exact / learned element energy \\
$L(g)$ & hypernetwork-generated factor, $K=L^\top L$ \\
$\lambda$ & regularization strength \\
$P_c,\,P_{\mathrm{rigid}}$ & projectors onto the physics nullspace \\
$\mathcal N_e$, $P_e$ & element nullspace; element-to-global scatter \\
$\mathbf 1$ & constant vector (scalar-physics nullspace) \\
$q(U),K(U)$ & element internal force / tangent stiffness \\
$\epsilon$ & energy accuracy $|E_\theta-E_{\mathrm{FEM}}|$ \\
$\mu,\mu_{\mathrm{glob}}$ & strong-convexity constants (element / assembled) \\
$N$ & number of assembled elements \\
$\beta$ & reaction coefficient (nonlinear benchmark) \\
$\eta$ & Dirichlet penalty weight ($10^6$) \\
\bottomrule
\end{tabular}
\end{table}

A tile's interaction with the rest of the model passes exclusively through
its $n=64$ boundary nodes. Partitioning the tile's P1 stiffness into
boundary (b) and interior (i) blocks and minimizing out the interior---
static condensation \citep{guyan1965reduction,craig1968coupling}---gives
the exact discrete tile energy as a quadratic form in the boundary values,
\begin{equation}
E_{\mathrm{FEM}}(g,U)=\min_{U_i}
\tfrac12\bigl(\begin{smallmatrix}U\\U_i\end{smallmatrix}\bigr)^{\!\top}
\!\bigl(\begin{smallmatrix}K_{bb}&K_{bi}\\K_{ib}&K_{ii}\end{smallmatrix}\bigr)\!
\bigl(\begin{smallmatrix}U\\U_i\end{smallmatrix}\bigr)
=\tfrac12\,U^\top S_g\,U,
\label{eq:schur}
\end{equation}
with $S_g=K_{bb}-K_{bi}K_{ii}^{-1}K_{ib}$ computed on a per-geometry mesh
(interior $\sim$550 nodes). Assembling the $S_g$ of adjacent tiles by
scatter-add and minimizing is \emph{exactly equivalent} to solving the
monolithic fine problem: condensation is energy-preserving, which is what
makes $S_g$ (rather than any field prediction) the correct learning target
for an element. $S_g$ is symmetric positive semidefinite with the
single nullvector $\mathbf{1}$ (pure-Neumann invariance under constant
shifts) and second eigenvalue $\lambda_2(S_g)\in[0.02,0.06]$ over the
geometry family $\mathcal G=[0.15,0.35]^2\times[0,\pi)$.

\subsection{Exact-condensation control}
Before analyzing learned elements we validate the assembly machinery itself:
assembling the \emph{exact} matrices $S_g$ for four different tiles into a
$1.6\times1.6$ domain and solving reproduces an independently meshed FEM
reference to $0.067\%$ mean relative $L^2$ (5 boundary-condition seeds), with
the Newton loop matching a direct solve to machine precision. All subsequent
error is therefore attributable to the learned energy alone.

\subsection{Why value-trained elements fail}
We take a DIMON-style field-predicting NOE trained to 0.34\% test error
(value regression over the same geometry family, 1{,}500 training samples
including boundary-data augmentation matched to multi-tile deployment; its
training budget, architecture width, and data generator match those used
for our elements, and its single-tile accuracy of 1.85\% mean is the
strongest we obtained for this family across all field-predicting variants
we trained) and assemble four instances. The induced energy's assembled
tangent, evaluated \emph{at the true solution}, has 36 negative eigenvalues;
started from zero the Newton loop reports converged step norms while landing
at 247\% error, and started from the truth it diverges. The failure is
structural, not an artifact of one training run: three further networks
trained from independent seeds (identical data, architecture, and budget;
single-tile field accuracy $0.32$--$0.36\%$) carry $31$--$44$ negative
eigenvalues at the true solution, land at $241$--$291\%$ error from zero
initialization, and, warm-started \emph{at the exact solution}, drift away
by three to six orders of magnitude within three Newton steps---every
seed, no exception. Value training
controls $\mathrm{NN}(BC,x)\approx T(x)$ pointwise but not
$\nabla_U E$ or $\nabla^2_U E$; the energy induced through the network's
predictions has neither its minimum at the true solution nor a
positive-semidefinite Hessian. This failure mode was reported as an open
obstacle in our earlier experiments and, we believe, explains why
NOEM deployments to date use a \emph{single} NOE coupled to standard FE
elements whose stiffness dominates the assembled tangent.

\section{Convex neural energy elements}\label{sec:method}

\subsection{Element definition}
A NIEL element type is a triple (reference cell topology, geometry space
$\mathcal G$, trained energy $E_\theta$), with $E_\theta$ as in
\eqref{eq:element}. The element exposes the standard FE contract
$q(U)=\nabla_U E_\theta$, $K(U)=\nabla^2_U E_\theta$ (by automatic
differentiation, or analytically in the quadratic case:
$K(g)=L(g)^\top L(g)+\lambda(I-P_c)$ with $P_c=\mathbf{1}\mathbf{1}^\top/n$).
Geometry enters through the smooth encoding
$[a,b,\cos2\theta,\sin2\theta]$ (the $2\theta$ frequency respects the
ellipse's $\pi$-periodicity); the hypernetwork is a plain MLP mapping this
encoding to the $n(n{+}1)/2$ entries of a symmetric $L(g)$.

Deployment (Algorithm~\ref{alg:assemble}) is classical FEM assembly:
element tangents and internal forces are scatter-added over shared boundary
nodes, Dirichlet data is imposed by a quadratic penalty
$\tfrac\eta2\sum_{j\in\mathcal D}(U_j-\bar u_j)^2$ with $\eta=10^6$ (any
standard alternative---elimination or Lagrange multipliers---works equally;
the penalty keeps the energy interpretation intact), and the assembled
convex problem is solved by Newton's method. For quadratic elements the
analytic-Hessian fast path makes each Newton step a single sparse solve and
the iteration converges in two steps; no network weights are updated at
deployment time.

\begin{algorithm}[t]
\caption{Deployment: monolithic assembly and solve}
\label{alg:assemble}
\begin{algorithmic}[1]
\Require trained element $E_\theta$; layout of $N$ tiles with geometries
$\{g_e\}$; global/local index maps $\{P_e\}$; Dirichlet data
$(\mathcal D,\bar u)$; penalty $\eta$
\For{$e=1,\dots,N$} \Comment{setup: one forward pass per geometry}
  \State $L_e \gets \mathrm{Hypernet}_\theta(g_e)$;\quad
         $K_e \gets L_e^\top L_e + \lambda(I-P_c)$
\EndFor
\State $U \gets 0$
\Repeat \Comment{Newton on the assembled convex energy}
  \State $K \gets \sum_e P_e^\top K_e(U|_e)\,P_e$;\quad
         $q \gets \sum_e P_e^\top q_e(U|_e)$
  \State add $\eta(U_j-\bar u_j)$ and $\eta$ to $q$ and $\mathrm{diag}(K)$
         for $j\in\mathcal D$
  \State solve $K\,\delta = -q$ (sparse); \quad $U \gets U + \delta$
\Until{$\|\delta\| < 10^{-10}$}
\State \Return $U$
\end{algorithmic}
\end{algorithm}

\subsection{Training}
For quadratic physics the energy is learned by \emph{operator regression}
(Algorithm~\ref{alg:train}):
the hypernetwork target is the symmetric square root
$M_g=\mathrm{sqrtm}\!\left(S_g-\lambda(I-P_c)\right)$, well defined precisely
when $\lambda\le\lambda_2(S_g)$---an explicit, checkable admissibility
condition for the regularization strength (pool minimum $0.0205$; we use
$\lambda=0.005$).

\begin{algorithm}[t]
\caption{Element training by operator regression (quadratic physics)}
\label{alg:train}
\begin{algorithmic}[1]
\Require geometry pool $\{g_k\}_{k=1}^{400}\subset\mathcal G$; physics
nullspace projector $P$ ($P_c$ or $P_{\mathrm{rigid}}$); strength $\lambda$
\For{$k=1,\dots,400$} \Comment{one-time target generation}
  \State mesh tile $g_k$; assemble $K$; condense
         $S_k \gets K_{bb}-K_{bi}K_{ii}^{-1}K_{ib}$ \hfill (Eq.~\ref{eq:schur})
  \State \textbf{assert} $\lambda \le \lambda_{\dim\ker+1}(S_k)$
         \Comment{admissibility check}
  \State $M_k \gets \mathrm{sqrtm}\big(S_k-\lambda(I-P)\big)$
\EndFor
\State train $\theta$:
$\min_\theta \sum_k \|\,\mathrm{sym}(\mathrm{Hypernet}_\theta(g_k)) - M_k\|_F^2$
\Statex \hspace{\algorithmicindent} (Adam, 20k steps, then L-BFGS polish;
65\,s wall-clock)
\State \Return $E_\theta(g,U)=\tfrac12\|L(g)U\|^2+\tfrac\lambda2\|(I-P)U\|^2$,
$L(g)=\mathrm{sym}(\mathrm{Hypernet}_\theta(g))$
\end{algorithmic}
\end{algorithm} This is markedly better conditioned than regressing
energies sampled at states: with 400 training geometries the hypernetwork
reaches a matrix mean-squared error of $5.5\times10^{-6}$ (relative
$2.7\times10^{-4}$, i.e., $1.6\%$ per-entry RMS). This residual sits below
the measured discretization sensitivity of the targets themselves:
regenerating $S_g$ with the interior mesh density perturbed by $\pm10\%$
changes the matrices by $3.5$--$7.3\%$ in Frobenius norm, so the network
fits the smooth cross-geometry structure at a level finer than any single
mesh realization is trustworthy. At an equal training budget, direct regression of sampled
energies onto the same architecture reaches an energy accuracy roughly
$2{,}500\times$ worse (manifold-restricted $\epsilon^2$ of $3\times10^{-1}$
vs.\ $1.2\times10^{-4}$) and assembles at 153.7\% error
(Section~\ref{sec:ablations})---sampled scalar energies are a poorly
conditioned supervision signal for operator-grade accuracy, which
retrospectively explains part of the difficulty of the naive approach.
For non-quadratic physics the PICNN correction is trained with a Sobolev
loss on energies and boundary forces, the latter available at no extra cost
via the envelope theorem applied to interior condensation
(Section~\ref{sec:nonlinear}).

\subsection{Theory}\label{sec:theory}
Write $\mu$ for the strong-convexity constant of the assembled energy on the
constrained subspace and $\epsilon$ for the energy accuracy
$|E_\theta - E_{\mathrm{FEM}}|$ on the relevant states. Full statements and
proofs are in the appendix; we summarize the chain.

\begin{theorem}[Energy accuracy controls solution accuracy]\label{thm:t1}
If $E_\theta$ is $\mu$-strongly convex on the constrained subspace and
$|E_\theta-E_{\mathrm{FEM}}|\le\epsilon$ at the true and model solutions,
then $\|U^*_\theta-U^*\|^2 \le 8\epsilon/\mu$.
\end{theorem}

\begin{theorem}[Assembly, general element nullspaces]\label{thm:global}
Let each element energy be convex with
$\nabla^2_U E_e \succeq \lambda\,(I-P_{\mathcal N_e})$, where
$P_{\mathcal N_e}$ projects onto a finite-dimensional admitted nullspace
$\mathcal N_e$, and let $P_e$ denote the element-to-global scatter map. If
the only global vector $V$ with $P_eV\in\mathcal N_e$ for every $e$ and
$V=0$ on the Dirichlet set is $V=0$, then the assembled energy is strictly
convex on the constrained subspace; with $N$ elements,
$\|U^*_\theta-U^*\|^2\le 4N\epsilon/\mu_{\mathrm{glob}}$.
\end{theorem}

Two corollaries instantiate the compatibility condition (proofs in the
appendix). \emph{(i)~Scalar diffusion}, $\mathcal N_e=\mathrm{span}\{\mathbf
1\}$: a connected element adjacency graph plus one pinned node suffice,
since a shared node forces neighboring element constants to agree; the
condition is per-element, so mixtures of element types
(Section~\ref{sec:hetero}) are covered with no further argument.
\emph{(ii)~Plane-strain elasticity}, $\mathcal N_e=\mathrm{span}\{r_x,r_y,
r_\theta\}$ (rigid modes): adjacent elements sharing at least two distinct
nodes force their infinitesimal rigid motions to coincide, and Dirichlet
data pinning two or more distinct nodes (a boundary segment in all our
experiments) annihilates the remaining global rigid motion.

\begin{theorem}[Architectural PSD]\label{thm:psd}
Every energy of the form \eqref{eq:element} has
$\nabla_U^2 E \succeq \lambda (I-P_c)$ for all $(g,U)$. Consequently every
Newton system on the assembled energy (with the compatibility condition of
Theorem~\ref{thm:global}) is uniquely solvable and the assembled problem
has a unique minimizer; for quadratic elements the iteration is one sparse
solve followed by a convergence check.
\end{theorem}

\begin{theorem}[Geometry generalization]\label{thm:t5}
Under Lipschitz dependence of $E_\theta$ and $E_{\mathrm{FEM}}$ on $g$,
$\|U^*_\theta-U^*\|^2(g_{\mathrm{test}}) \le 8\epsilon/\mu +
c\,d(g_{\mathrm{test}},\mathcal G_{\mathrm{train}})/\mu$ with
$c = 8(L_\theta+L_F)$; the observed exponent $\approx0.9$ matches this
linear-in-$d$ form.
\end{theorem}

\begin{remark}[Nullspace principle]\label{rem:nullspace}
If the regularizer's nullspace does not contain the physics nullspace, the
achievable $\epsilon$ is bounded below by
$\min_{\alpha,\beta}\mathbb E[(\tfrac\lambda2 c^2+\beta c+\alpha)^2]^{1/2}$
where $c$ is the physics-nullspace coordinate of the data. We verify this
floor to within measurement accuracy across three decades of $\lambda$
(Section~\ref{sec:ablations}).
\end{remark}

The constants are measured, not just posited: $\mu_{\mathrm{glob}}$ is the
second eigenvalue of the assembled tangent ($0.018$ for $2\times2$,
decaying with domain size exactly as the classical Poincar\'e constant
does), and every experimental point in Section~\ref{sec:experiments} is
checked against its bound.

\section{Experiments}\label{sec:experiments}

All assembly errors are relative $L^2$ norms over the \emph{skeleton}
(element-boundary) degrees of freedom---the variables the assembled system
solves for---measured against the exact-condensation solve on the
same discretization (isolating learned-energy error; the control of
Section~\ref{sec:failure} bounds the additional mesh-difference error at
0.07\%), and geometries used in evaluation are excluded from training.
Field visualizations recover element interiors by exact lifting applied
identically to both solutions; the cost accounting of full-field recovery
is treated separately in Section~\ref{sec:cost}.

\subsection{Library milestone}
A single trained element assembles four unseen geometries into a
$1.6\times1.6$ domain at $\mathbf{0.72\pm0.33\%}$ relative $L^2$
(mean $\pm$ s.d.\ over 50 runs: five held-out geometry assignments $\times$
ten boundary conditions; median $0.62\%$, max $1.67\%$), Newton
terminating in 2 iterations in every run---for these quadratic energies
that means one sparse linear solve plus a convergence-verification step
(Figure~\ref{fig:fields}). The result is also stable under training
stochasticity: five independently initialized and trained networks
(identical data, architecture, and optimization) evaluated on one fixed
50-run suite give mean errors of $0.58$--$0.72\%$ (standard deviation of
the means $0.05\%$), with every run of every seed converging in the same
2 iterations.

\begin{figure*}[t]
\centering
\includegraphics[width=\textwidth]{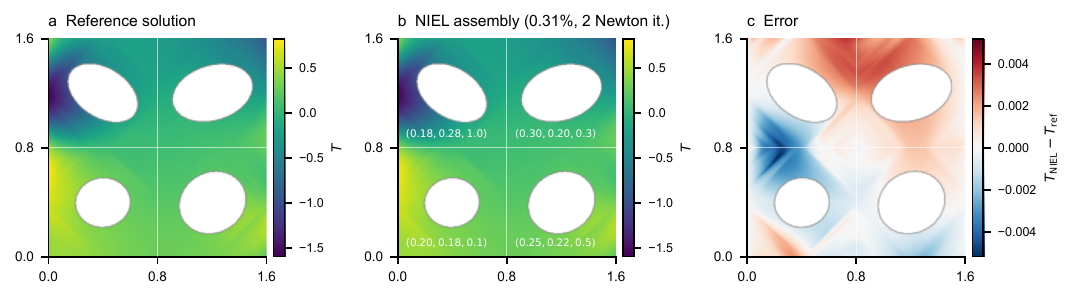}
\caption{Four convex neural energy elements with \emph{unseen} geometries
assembled into a $1.6\times1.6$ domain. (a)~Reference solution;
(b)~assembled solution ($0.31\%$ relative $L^2$, 2 Newton iterations), with
each element's held-out geometry $(a,b,\theta)$ annotated;
(c)~pointwise error $T_{\mathrm{NIEL}}-T_{\mathrm{ref}}$ (absolute, same
units as $T$; the field spans $[-1.6,0.8]$), concentrated at element
interfaces where the learned energies meet. White lines mark element
boundaries; holes are masked.}
\label{fig:fields}
\end{figure*}

\subsection{Theory validation point-by-point}
\begin{figure*}[t]
\centering
\includegraphics[width=\textwidth]{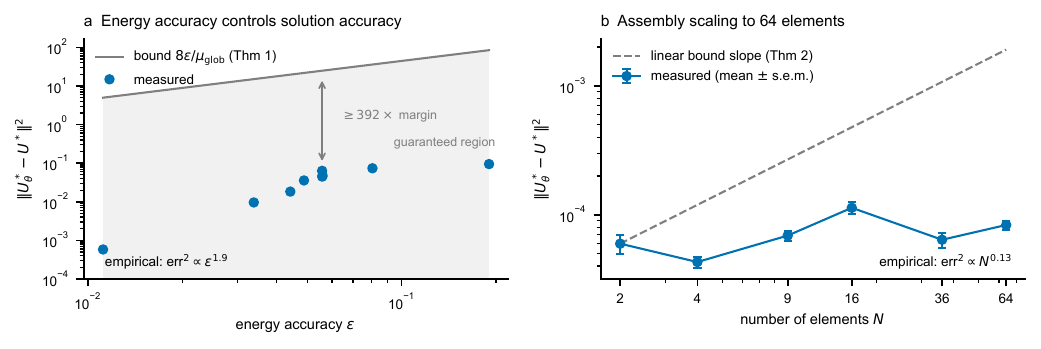}
\caption{Theory validated point-by-point. (a)~Every model-quality
checkpoint satisfies the Theorem~\ref{thm:t1} bound (gray line; shaded
region is guaranteed), with two to three orders of magnitude of margin.
The empirical rate ($\mathrm{err}^2\propto\epsilon^{1.9}$) is faster than
the guaranteed linear rate because the worst case of the proof---the full
energy gap concentrating in the softest mode of the assembled
tangent---does not occur for generic geometry and boundary-condition draws.
(b)~Assembly error scaling across $2$--$64$ elements (40 runs per $N$,
mean $\pm$ s.e.m.) stays far below the linear bound of
Theorem~\ref{thm:global}.}
\label{fig:theory}
\end{figure*}
\textbf{(T1)} Across nine model-quality checkpoints spanning a $17\times$
range of $\epsilon$, every point satisfies the Theorem~\ref{thm:t1} bound
(Figure~\ref{fig:theory}a); the empirical rate is
$\mathrm{err}^2\propto\epsilon^{1.9}$, i.e., faster than the guaranteed
linear rate.
\textbf{(T2)} Across $N\in\{2,4,9,16,36,64\}$ tiles (111--2{,}241 unknowns;
40 runs per $N$ over four random geometry assignments each), mean errors
stay within $0.63$--$1.01\%$ with no systematic growth,
i.e., $\mathrm{err}^2\propto N^{0.13}$, far
within the linear bound; the assembled tangent has exactly one near-null
mode (the global constant) at every $N$, its second eigenvalue decays with
domain size exactly as the classical Poincar\'e constant prescribes
($0.018$ at $2\times2$ to $0.0021$ at $8\times8$), and Newton takes 2
iterations throughout.
\textbf{(T3)} The convex element shows zero negative tangent eigenvalues at
both the zero state and the solution; the value-trained control shows 1 and
36 respectively (Section~\ref{sec:failure}).
\textbf{(T5)} Extrapolating the hole semi-axes up to $0.03$ outside the
training box degrades errors gracefully ($4.5\%$ worst case at $a=0.38$,
hole nearly touching the tile boundary); the bound holds with orders of
magnitude of margin in its constant, and the observed exponent
$\approx0.9$ matches the linear-in-$d$ form of Theorem~\ref{thm:t5}.
\begin{figure*}[t]
\centering
\includegraphics[width=\textwidth]{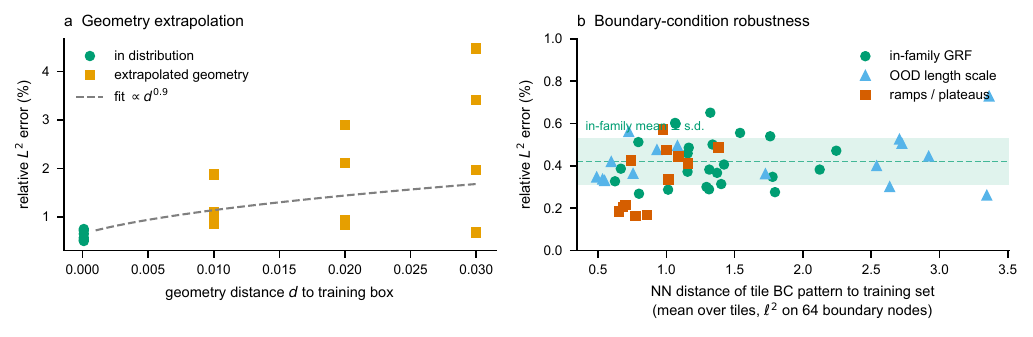}
\caption{Out-of-distribution behavior. (a)~Geometry extrapolation beyond
the training box degrades gracefully (worst case $4.5\%$ with the hole
nearly touching the element boundary); the Theorem~\ref{thm:t5} bound
holds with large margin. (b)~Boundary-condition families far outside the
training distribution (ramps, plateaus) are indistinguishable from
in-family performance (green band: in-family mean $\pm$ s.d.). The
abscissa is the mean over tiles of the nearest-neighbor $\ell^2$ distance
between the tile's boundary values at the exact solution and the training
set of boundary patterns; the element stores an operator, not a response
to a data distribution.}
\label{fig:ood}
\end{figure*}
\textbf{(BC robustness)} Far out-of-distribution boundary data---linear
ramps and plateau bumps, not draws from the training GRF family---incur
errors of $0.16$--$0.57\%$, indistinguishable from in-distribution
performance (Figure~\ref{fig:ood}b). The element learns an \emph{operator} (a quadratic form), not a
response to a boundary-data distribution; robustness is structural, in sharp
contrast to field-predicting NOEs whose accuracy collapsed from 2.5\% to
60--85\% under boundary-pattern shift in our earlier experiments.

\subsection{Non-rectangular topology: an L-shaped assembly}\label{sec:lshape}
Nothing in the element or the theory is specific to rectangular layouts:
Theorem~\ref{thm:global} asks only that the element adjacency graph be
connected and that Dirichlet data pin the constant mode. We exercise this
on an L-shaped domain---a $4\times4$ grid with the top-right $2\times2$
block removed (12 tiles, re-entrant corner, non-convex boundary)---with
held-out geometries and smooth random Dirichlet data on the full exterior
boundary. Over 50 runs (5 geometry draws $\times$ 10 boundary conditions)
the assembly reaches $0.40\pm0.09\%$ relative $L^2$ (max $0.57\%$), Newton
again converging in 2 iterations, and the assembled tangent again carries
exactly one near-null mode (Figure~\ref{fig:lshape}). The re-entrant
corner, where three tiles meet and the solution gradient is largest,
produces no error concentration beyond the usual interface signature. The
same trained element, the same assembly code, a different topology: this
is the element-library contract in action.

\begin{figure*}[t]
\centering
\includegraphics[width=\textwidth]{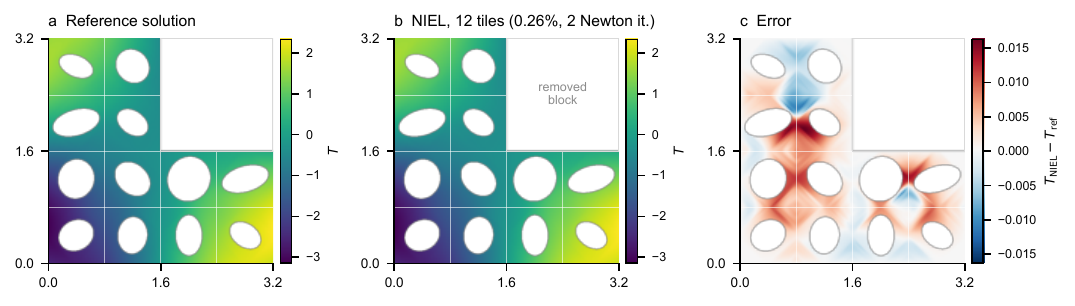}
\caption{Non-rectangular topology. An L-shaped domain (a $4\times4$ grid
with the top-right $2\times2$ block removed; 12 elements, all unseen
geometries) assembled from the same trained element used everywhere else
in the paper. (a)~Reference solution (exact condensation on the same
layout); (b)~NIEL assembly ($0.26\%$ relative $L^2$ on this instance,
2 Newton iterations; $0.40\pm0.09\%$ over 50 runs); (c)~pointwise error,
showing the usual interface signature and no concentration at the
re-entrant corner. White lines mark element boundaries; holes and the
removed block are masked.}
\label{fig:lshape}
\end{figure*}

\subsection{Heterogeneous element types}\label{sec:hetero}
The library claim requires more than one element type. We therefore train a
second type on the same tile footprint: \emph{two circular holes} of radii
$r_1\ge r_2\in[0.08,0.15]$ centered at $\pm0.17(\cos\varphi,\sin\varphi)$
from the tile center (encoding $[r_1,r_2,\cos\varphi,\sin\varphi]$; a
different hole topology, hence a genuinely different condensed-operator
family). The recipe is identical to the first type---operator regression at
the same $\lambda=0.005$ (admissible: pool-minimum
$\lambda_2=0.047$), same architecture and budget, $78$\,s total one-off
cost---and the two trained types are then \emph{mixed freely in one
monolithic assembly}, each tile drawing its stiffness from its own type's
hypernetwork. Truth is the exact mixed-Schur assembly of the same layout.
Mixed $2\times2$ assemblies (checkerboard and random type assignments,
held-out geometries) reach $1.13\pm0.44\%$ over 50 runs; mixed L-shaped
12-element assemblies reach $\mathbf{0.46\pm0.11\%}$ (max $0.71\%$,
Figure~\ref{fig:hetero}); the Newton path converges in 2 iterations with
exactly one near-null tangent mode, as before. No component of the
framework---assembly, theory, or regularization---distinguishes the types:
Theorem~\ref{thm:global}'s hypotheses are per-element, so the guarantee is
\emph{type-agnostic by construction}. This is the library mechanism: adding
an element type is one training run, and deployment composes types without
retraining.

\begin{figure*}[t]
\centering
\includegraphics[width=\textwidth]{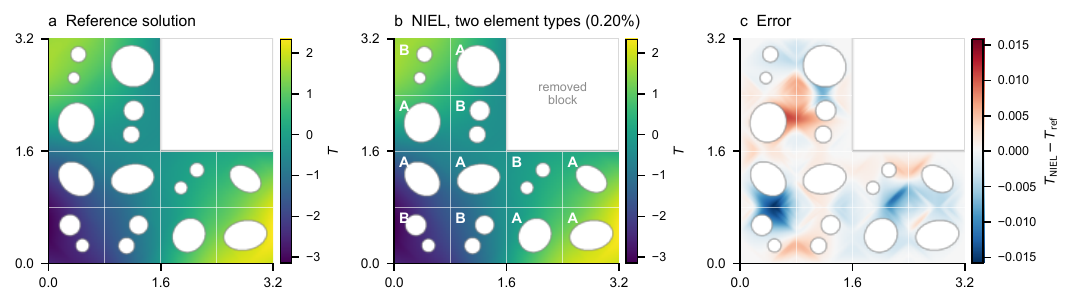}
\caption{Heterogeneous element types. An L-shaped 12-element assembly
mixing two trained element types (labels in panel b): type A, one elliptic
hole; type B, two circular holes---a different hole topology with its own
geometry parameterization and hypernetwork, trained independently at the
same $\lambda$. (a)~Reference (exact mixed-Schur assembly);
(b)~NIEL assembly ($0.20\%$ relative $L^2$ on this instance;
$0.46\pm0.11\%$ over 50 runs); (c)~pointwise error: the interface
signature, with no visible artifact at A--B interfaces relative to A--A or
B--B ones.}
\label{fig:hetero}
\end{figure*}

\subsection{Second physics: plane-strain elasticity with a
three-dimensional nullspace}\label{sec:elast}

The scalar benchmark exercises the nullspace principle only in its simplest
instance (nullspace $=$ constants). We therefore repeat the construction
for plane-strain linear elasticity ($E{=}1$, $\nu{=}0.3$) on the same tile
family: traction-free hole boundary, $n=128$ interleaved displacement
degrees of freedom on the 64 boundary nodes. The condensed stiffness
$S_g\in\R^{128\times128}$ is positive semidefinite with \emph{exactly
three} zero modes (two rigid translations and the infinitesimal rotation;
verified to $10^{-14}$ per geometry), and its first physical eigenvalue
$\lambda_4(S_g)$ reaches $7\times10^{-4}$ at the pool minimum---so the
admissibility condition $\lambda\le\lambda_4$ is far tighter than in the
scalar case and we deploy $\lambda=5\times10^{-4}$.

\emph{Assembly machinery transfers.} Assembling four exact elasticity
Schur complements reproduces an independently meshed plane-strain FEM
reference to $0.047\%$ mean relative $L^2$ (5 boundary-condition seeds,
max $0.069\%$), and the assembled tangent restricted to the constrained
subspace is positive definite (smallest eigenvalue $0.055$, zero negative
eigenvalues) even though every element now carries a three-dimensional
kernel---the compatibility condition of Theorem~\ref{thm:global},
instantiated for rigid modes by corollary~(ii): shared edges force
neighboring rigid motions to coincide, and the pinned boundary segment
annihilates the remaining global one.

\emph{The nullspace principle discriminates between three regularizers.}
With a 3-dim physics nullspace the principle admits graded violations:
gauge$_3$ ($\lambda(I-P_{\mathrm{rigid}})$, nullspace covers all three
modes), gauge$_2$ (covers the translations but \emph{misses the
rotation}---a failure mode that cannot exist in scalar physics), and full
Tikhonov ($\lambda I$). Figure~\ref{fig:elast}a isolates the resulting
bias from all learning error by assembling \emph{best-case} elements (the
clipped regression targets themselves): gauge$_3$ is exact until $\lambda$
exceeds the benchmark tiles' $\lambda_4$ ($0.0098$--$0.024$), while
gauge$_2$ and full Tikhonov exhibit irreducible floors at every $\lambda$,
ordered by how much of the nullspace the regularizer misses (at
$\lambda=5\times10^{-3}$: $0\%$ vs.\ $0.85\%$ vs.\ $6.7\%$). Trained
elements inherit these floors additively (Figure~\ref{fig:elast}b): at the
deployed $\lambda$ the full-Tikhonov model assembles at $0.71\%$, matching
its training-free floor of $0.73\%$ almost exactly, while gauge$_3$
reaches $\mathbf{0.20\pm0.10\%}$ on 50 held-out-geometry assemblies
(5 geometry draws $\times$ 10 boundary conditions, max $0.63\%$)---the
same accuracy class as the scalar benchmark, at identical hypernetwork
relative accuracy ($10^{-3}$) and 298\,s training wall-clock.
Figure~\ref{fig:elast_fields} shows the corresponding displacement fields:
the error is again an interface signature, now in a vector-valued field
whose element energies each carry a three-dimensional kernel.

\begin{figure*}[t]
\centering
\includegraphics[width=\textwidth]{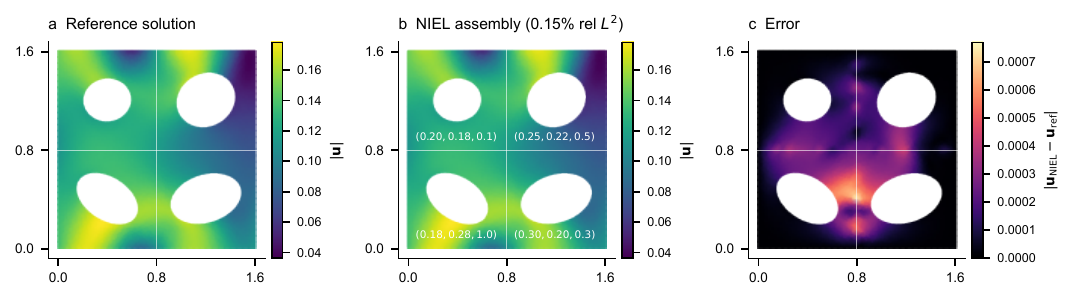}
\caption{Elasticity fields. A four-element plane-strain assembly of unseen
geometries with the gauge$_3$ element ($\lambda=5\times10^{-4}$).
(a)~Reference displacement magnitude $|\mathbf u|$ (exact condensation);
(b)~NIEL assembly ($0.15\%$ relative $L^2$ on this instance), with each
element's held-out geometry annotated; (c)~pointwise displacement error,
concentrated at element interfaces. Element interiors are recovered by the
exact elastic lifting of the boundary values (one interior solve per
element, applied identically to both solutions); holes are traction-free.}
\label{fig:elast_fields}
\end{figure*}

\begin{figure*}[t]
\centering
\includegraphics[width=\textwidth]{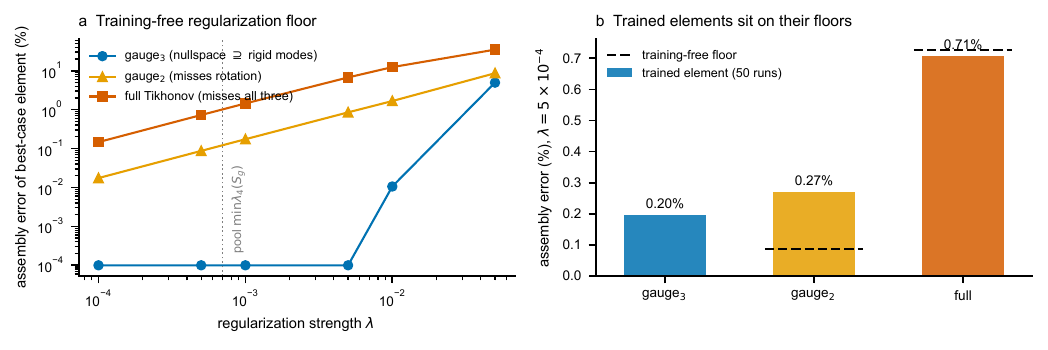}
\caption{Plane-strain elasticity: the nullspace principle with a
three-dimensional physics nullspace (rigid modes). (a)~Assembly error of
\emph{training-free best-case} elements (clipped targets) as a function of
$\lambda$, isolating the regularization bias from all learning error: the
regularizer that covers the full rigid-mode nullspace (gauge$_3$) is
exact; missing only the rotation (gauge$_2$) already creates an
irreducible floor; full Tikhonov is worst. Dotted line: pool-minimum
$\lambda_4(S_g)$, the admissibility threshold of
Section~\ref{sec:method}. (b)~Trained hypernetwork elements at
$\lambda=5\times10^{-4}$ (identical architecture, data, and optimization;
only the regularizer differs) inherit the floors: the full-Tikhonov error
matches its training-free floor (dashed), while gauge$_3$'s error is pure
regression error.}
\label{fig:elast}
\end{figure*}

\subsection{A three-dimensional instantiation}\label{sec:3d}
Nothing in the construction is intrinsically two-dimensional, and 3D is
where the economics sharpen: interior degrees of freedom grow much faster
than boundary ones, so condensation compresses more, and per-geometry
meshing costs more. We instantiate a 3D heat-conduction element on a cube
tile ($[0,0.8]^3$) with an interior axis-aligned ellipsoidal void, semi-axes
$(a,b,c)\in[0.15,0.3]^3$. The tile's six faces carry structured
(transfinite) $9\times9$ face grids, so every geometry exposes the same
386 canonical surface nodes and tiles are assembly-compatible by
construction; the interior is meshed freely and refined near the void
($\sim$1{,}000 nodes per tile, condensed away). The recipe is unchanged:
operator regression of
$M_g=\mathrm{sqrtm}(S_g-\lambda(I-P_c))\in\R^{386\times386}$ at
$\lambda=0.002$ (pool-minimum $\lambda_2=0.0095$), a $3\to256\to256\to
386^2$ hypernetwork (38M parameters), 400 training geometries.

The exact-condensation control transfers: assembling eight exact
condensed cubes into a $1.6^3$ domain reproduces an independently meshed
monolithic 3D FEM reference (26{,}300 tetrahedra) to $0.69\%$ mean
relative $L^2$ (5 boundary seeds; the coarser 3D discretizations dominate
this gap), with the constrained assembled tangent positive definite. The
trained element assembles $2\times2\times2$ configurations of eight
\emph{unseen} geometries at $\mathbf{0.23\pm0.06\%}$ relative $L^2$
against the exact-condensation truth (50 runs: 5 geometry draws $\times$
10 boundary conditions, max $0.42\%$; Figure~\ref{fig:threed} shows one
instance as a 3D cutaway on the internal shared tile-face planes), and
the accuracy is flat in element count:
$0.18\pm0.05\%$ at $N=4$ ($2\times2\times1$), $0.22\pm0.05\%$ at $N=8$,
$0.24\pm0.06\%$ at $N=27$ ($3\times3\times3$, 6{,}364 assembled
unknowns). Per new geometry, 3D mesh-and-condense takes $188$\,ms
against a $6.3$\,ms hypernetwork forward ($30\times$; the modest network
is $386^2$-output and CPU-bound, and the 3D meshes here are deliberately
coarse---both factors that scale in the element's favor with resolution).
The library mechanism, the admissibility check, and the assembly
guarantee carry over verbatim.

The honest cost of the dense-operator route is its output scaling: the
hypernetwork emits $386^2\approx1.5\times10^5$ entries (38.4M parameters,
307\,MB checkpoint, $1.2$\,MB per instantiated element stiffness), all
$O(n_b^2)$ in the number of interface nodes $n_b$, with the
$\mathrm{sqrtm}$ target generation $O(n_b^3)$. This is affordable at the
present interface resolution but will not scale as-is to fine 3D
interfaces; low-rank or spectral-basis output parameterizations of
$L(g)$ are the natural remedy and preserve the convex reconstruction
$L^\top L$, which is what the guarantees actually use.

\begin{figure*}[t]
\centering
\includegraphics[width=\textwidth]{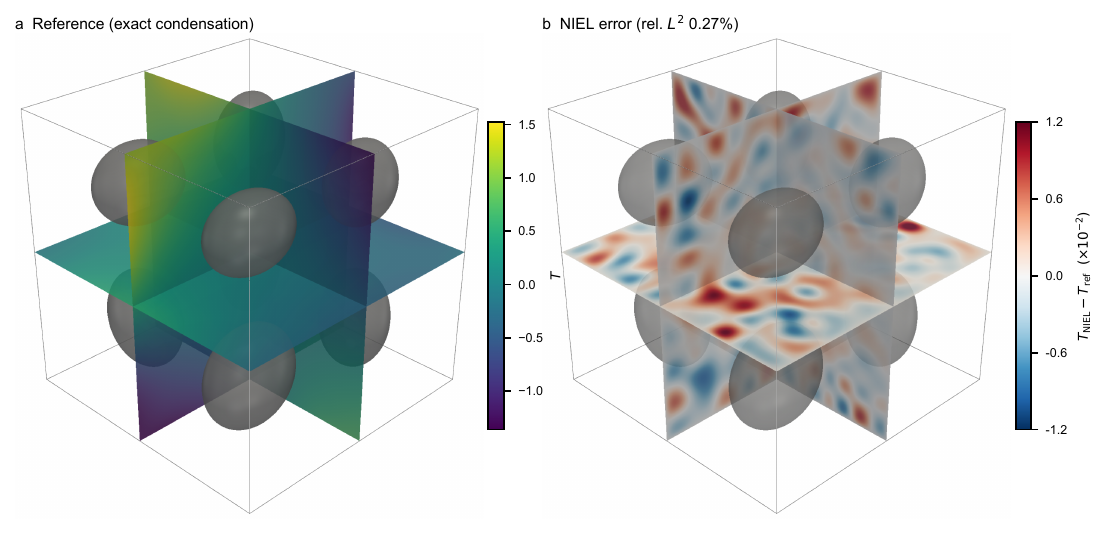}
\caption{Three-dimensional instantiation. A $2\times2\times2$ assembly of
eight cube elements with unseen ellipsoidal-void geometries
$(a,b,c)\in[0.15,0.3]^3$, solved with the trained 3D element
($\lambda=0.002$). Cutaway view on the three internal shared tile-face
planes $x=0.8$, $y=0.8$ and $z=0.8$, which carry assembled degrees of
freedom; the eight ellipsoidal voids (translucent) are strictly interior
and do not intersect these planes. (a)~Reference solution (exact
condensation); (b)~pointwise error of the NIEL assembly on the same
planes ($0.27\%$ relative $L^2$ on this instance; $0.23\pm0.06\%$ over 50
runs).}
\label{fig:threed}
\end{figure*}

\subsection{Ablations}\label{sec:ablations}
\begin{table}[t]
\centering\footnotesize
\caption{Ablations and baselines. Mean relative $L^2$ over the held-out
assembly suite (15 runs).}
\label{tab:ablation}
\setlength{\tabcolsep}{3pt}
\resizebox{\columnwidth}{!}{
\begin{tabular}{llr}
\toprule
Family & Configuration & Error \\
\midrule
\textbf{Ours} & quad-sqrt hypernet, gauge, $\lambda=0.005$ & \textbf{0.66\%} \\
Regularizer strength & gauge, $\lambda=0.01$ & 0.70\% \\
Regularizer nullspace & \emph{full} Tikhonov, $\lambda=0.005$ & 2.17\% \\
Training route & sampled-energy regression & 153.7\% \\
Interpolation of $S(g)$ & nearest neighbor (400 anchors) & 3.26\% \\
Interpolation of $S(g)$ & RBF, $k{=}32$ & 5.64\% \\
Regression of $M(g)$ & nearest neighbor (same anchors) & 3.26\% \\
Regression of $M(g)$ & RBF, $k{=}32$ & 5.74\% \\
Regression of $M(g)$ & polynomial ridge, degree 3 & 0.44\% \\
Whole-domain surrogate & MLP, $2000$ assemblies ($4\times$ data) & 8.14\% \\
Value-trained element & field NOE (Section~\ref{sec:failure}) & 247\% \\
\bottomrule
\end{tabular}}
\end{table}

The full-vs-gauge line isolates the nullspace principle: identical
architecture, data, and optimization; only the regularizer's nullspace
differs; the error triples, and the training loss sits exactly on the
analytic floor of Remark~\ref{rem:nullspace}. The interpolation lines
instantiate parametric operator interpolation in the spirit of
\citet{amsallem2011online} with the same 400 anchors available to our
hypernetwork. To make the comparison variable-fair we also run the same
classical methods on the \emph{same regression target} as the
hypernetwork---the square-root factor $M(g)$, so that
$\widehat S=\widehat M^\top\widehat M+\lambda(I-P_c)$ is positive
semidefinite by construction for any interpolant. Local methods do not
improve in the fairer variable (nearest-neighbor $3.26\%$, RBF $5.74\%$):
the bottleneck is locality, not the variable. A \emph{global} degree-3
polynomial ridge fit of $M(g)$, by contrast, reaches $0.44\%$ on this
suite---slightly better than the hypernetwork---which sharpens rather
than weakens the paper's point: on a smooth, low-dimensional,
single-topology family, \emph{any} adequate global regressor of the
square-root factor yields a working convex element. The load-bearing
contributions are the choice of learning target and its convex
reconstruction, the nullspace-compatible regularization, and the assembly
theory; the hypernetwork is one convenient global regressor (and the one
that extends to richer encodings), not the source of the guarantees. The
whole-domain line makes the compositionality argument quantitative:
learning the $12$-dimensional product geometry space directly, with four
times more training assemblies, is $12\times$ worse than learning the
3-dimensional element family once and assembling.

\subsection{Cost and amortization}\label{sec:cost}
Per new geometry, element setup is a hypernetwork forward pass
($215\,\mu$s) vs.\ mesh-and-condense ($38$\,ms): $175\times$, the relevant
factor for geometry sweeps and inverse design. Per solve (four tiles, 189
unknowns, analytic-Hessian path), $17$\,ms vs.\ $184$\,ms for the reference
FEM mesh-and-solve ($\sim$$11\times$), with a $20\times$ reduction in
unknowns (3721 to 189).

To make the amortization argument concrete rather than notional, we
measured the full end-to-end cost of a design sweep: $M=50$ four-tile
evaluations, each with four fresh geometries (never seen in training,
no cached meshes), timed along three paths on the same machine
(Figure~\ref{fig:amort}). Per evaluation, the complete NIEL path
(four hypernetwork forwards, assembly, direct solve) takes $8.1$\,ms
against $133$\,ms for mesh-and-condense and $171$\,ms for monolithic FEM
($16$--$21\times$), at $0.64\%$ mean error against the exact-condensation
truth on the same sweep (max $1.18\%$)---the deployed accuracy, measured
inside the timing loop. Charging NIEL its entire one-time cost
($13$\,s to condense the 400 training targets plus $65$\,s of training,
$78$\,s total), the cumulative cost crosses over at
$M^*=627$ evaluations against mesh-and-condense and $M^*=482$ against
monolithic FEM---about $2{,}500$ tile geometries, i.e., a single medium
design study or a few dozen optimization iterations with population-based
methods. We state plainly that for a single linear solve the advantage is
modest and the one-off cost dominates below $M^*$; the value concentrates
in amortization across geometries and in the degrees-of-freedom
compression that grows with interior complexity, and both factors improve
with the interior/boundary ratio in 3D.

\paragraph{What is accelerated, and what is recovered.} The NIEL solve
returns the skeleton (element-boundary) degrees of freedom; all assembly
errors in this paper are computed on those DOFs, and the field figures
recover interiors by exact lifting $U_i=-K_{ii}^{-1}K_{ib}U_b$ applied
identically to both solutions. Three deployment regimes must therefore be
priced separately, and we measured all of them on the same sweep
(10 fresh four-tile cases). \emph{(i)~Boundary/energy quantities}
(compliance, stored energy, interface fluxes, boundary response---the
quantities that drive most design loops): the boundary-only pipeline
applies, $8.2$\,ms per case, $16$--$21\times$ faster than the classical
paths; every speedup and crossover claimed above belongs to this regime.
\emph{(ii)~Full interior fields via exact lifting}: the lifting needs the
tile mesh, assembly, and one factorization$+$backsolve per tile---costs
the boundary pipeline never pays---and the advantage disappears: NIEL
with lifting takes $127$\,ms per case against $128$\,ms for
mesh-and-condense (both still $1.3\times$ faster than the $171$\,ms
monolithic solve, since the skeleton system is smaller). For full-field
tasks the honest statement is parity with substructuring, not
acceleration. \emph{(iii)~Mesh-free full fields} would require a trained
lifting operator $(g,U_b,x)\mapsto U(x)$---a standard field-regression
task that does not participate in assembly (and so needs no convexity),
which we have not trained; we flag it as the natural companion model.

\begin{figure}[t]
\centering
\includegraphics[width=\columnwidth]{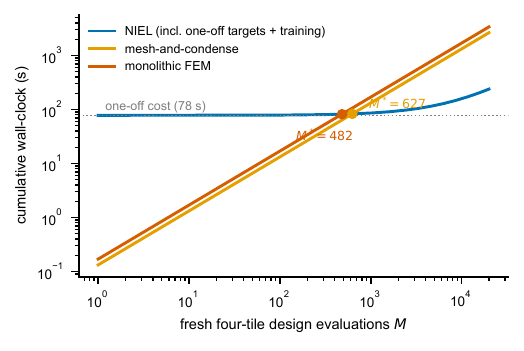}
\caption{Measured amortization. Cumulative wall-clock of a design sweep
vs.\ the number of fresh four-tile evaluations $M$, extrapolated from 50
measured cases per path. The NIEL line includes its full one-off cost
(condensing 400 training targets $+$ training). Crossovers: $M^*=627$
against mesh-and-condense, $M^*=482$ against monolithic FEM.}
\label{fig:amort}
\end{figure}

\section{Nonlinear extension}\label{sec:nonlinear}
We repeat the construction for $-\Delta T+\beta T^3=0$ ($\beta=100$), whose
tile energy adds a convex quartic potential; the condensed boundary energy
(partial minimization over the interior) is convex but non-quadratic, the
quartic contributes 6--21\% of the energy over the working range, and
ignoring the nonlinearity shifts solutions by 20--31\%. The physics
nullspace here is empty (the reaction pins constants), so per
Remark~\ref{rem:nullspace} an ordinary small Tikhonov term is admissible:
the nullspace \emph{principle}, not a fixed recipe, governs.

\paragraph{The framework transfers exactly.} Assembling four elements that
evaluate the \emph{exact} condensed nonlinear energy (interior Newton per
call; gradients by the envelope theorem, verified to $10^{-8}$; tangent by
condensed Hessians) reproduces the monolithic nonlinear solution to machine
precision ($0.0000\%$, 6--8 global Newton iterations). Monolithic assembly
of condensed convex energies is therefore exact for nonlinear physics, not
merely for quadratic ones.

\paragraph{The correction is convex.} Writing the condensed energy as
[known quadratic Schur part] $+$ [known boundary quartic] $+\;R(U)$, the
interior-mediated correction $R$ is nonnegative and \emph{convex}: adding a
convex interior potential increases the interior Hessian block, and the
condensed Hessian is monotone in it (Schur-complement monotonicity), so
$\nabla^2 R = H_F - S \succeq 0$; we verify the smallest eigenvalue of
$H_F-S$ is nonnegative to machine precision at sampled states---and observe
it is frequently \emph{near zero}, i.e., $R$ has strongly anisotropic,
near-singular curvature.

\paragraph{A training-free ladder of convex trial energies.} Fixing the
interior to its exactly-known \emph{linear} response $U_i = A U_b$,
$A=-K_{ii}^{-1}K_{ib}$, and evaluating the exact energy functional on this
trial state yields a closed-form convex element (a Rayleigh--Ritz upper
bound: quadratic $+$ quartic composed with a linear map). Refining the
trial interior by $k$ Newton steps tightens the bound at one sparse solve
per step. We make the refinement \emph{variationally consistent} by
differentiating through the Newton steps: the element energy
$E_k(U_b)=E_{\rm full}(U_b,U_i^k(U_b))$ is a single scalar function of
$U_b$, and its force and tangent are taken as its exact gradient and
Hessian by automatic differentiation (finite-difference checked to
$10^{-8}$; the $k=0$ rung reproduces the closed-form Ritz derivatives to
$10^{-16}$). On the four-tile benchmark (nine boundary-condition cases,
monolithic nonlinear truth), see Table~\ref{tab:ladder}.

\begin{table}[t]
\centering\footnotesize
\caption{Nonlinear benchmark: a ladder of convex trial energies. All rows
are convex elements; all are training-free unless noted.}
\label{tab:ladder}
\setlength{\tabcolsep}{3pt}
\resizebox{\columnwidth}{!}{
\begin{tabular}{lrr}
\toprule
Element & Error & Newton it.\\
\midrule
linear physics (nonlinearity ignored) & 25.7\% & 2 \\
$+$ boundary quartic (local terms only) & 14.6\% & 5--7 \\
linear-response Ritz ($k=0$, closed form) & 11.6\% & 7--9 \\
differentiable refinement ($k=1$) & \textbf{1.17\%} & 7 \\
differentiable refinement ($k=2$) & \textbf{0.040\%} & 7 \\
differentiable refinement ($k=3$) & \textbf{0.0002\%} & 7 \\
exact condensation ($k\to\infty$) & 0.000\% & 6--8 \\
\midrule
learned reduced-basis convex correction$^\dagger$ & 4.9\% & 5--7 \\
\bottomrule
\end{tabular}}
\end{table}

{\small $^\dagger$Trained; see \ref{app:negative}. An earlier
quasi-Newton variant of the $k=1$ rung---boundary force from the envelope
formula at the refined interior but tangent from the trial-space formula,
so $q$ and $K$ were not exact derivatives of a single scalar
energy---reached $1.34\%$ at $17$ iterations; variational consistency both
lowers the error and restores iteration parity with exact condensation.
The differentiable $k\ge1$ energies are not proved convex off the trial
manifold, but their assembled tangents are positive definite at the
converged states in all runs (smallest sampled Hessian eigenvalue
$+0.18$).}

\paragraph{Learning the correction: partial success, open gap.} Generic
Sobolev-trained convex networks for $R$ on the full boundary space (with
envelope-theorem gradient labels, remainder-targeted losses, and relative
weighting) consistently underperform even the no-learning physics base:
the near-singular anisotropic curvature documented above is a poor match
for smooth input-convex architectures, and gradient errors inject
directly into the assembled force balance (\ref{app:negative}).
The diagnosis itself, however, suggests a remedy: if the curvature is
concentrated in a few directions, learn only there. Restricting the
correction to the top $r=18$ eigen-directions of the remainder-gradient
covariance ($99.9\%$ of its variance; a convex network of $PU$ is convex
in $U$) and training gradient-first yields the first learned element to
outperform the physics base---$4.9\%$ vs.\ $14.6\%$ mean assembly error,
with positive-definite tangents throughout (smallest Hessian eigenvalue
$+0.11$). The gap to the physics ladder ($1.17\%$ at $k=1$) remains, and
checkpoint behavior confirms the value-versus-derivative tension: longer
training improves the pointwise gradient loss while \emph{degrading}
assembly. Closing the gap between learned corrections and the ladder
rungs---ideally with geometry parameterization inherited from the linear
construction---is a precisely-posed direction for future work.

\section{Related work}\label{sec:related}

\subsection{Operator learning and the deployment gap}
Neural operators learn maps between function spaces: DeepONet
\citep{lu2021deeponet} from the universal approximation theorem for
operators, Fourier neural operators \citep{li2021fno} through spectral
convolutions, with a unifying framework and approximation theory in
\citet{kovachki2023neural} and error estimates in
\citet{lanthaler2022error}; systematic empirical comparisons are given by
\citet{lu2022fair}. Geometric variability---the property an element library
must expose as an explicit parameter---is handled either by learned
coordinate transformations to a reference domain, as in Geo-FNO
\citep{li2023geofno} and DIMON \citep{yin2024dimon}, or by mesh-based graph
networks \citep{pfaff2021meshgraphnets}. In all of these lines the trained
operator is deployed as a \emph{standalone surrogate} for one domain (or one
diffeomorphic family): changing the problem layout---the number,
arrangement, or connectivity of components---means retraining. The gap this
paper addresses is precisely compositional reuse: one trained object,
instantiated many times, coupled through a solver.

\subsection{Physics-informed and variational training}
Physics-informed neural networks \citep{raissi2019pinn,karniadakis2021piml}
impose the PDE residual on a per-instance basis; variational formulations
replace the residual by an energy, from the Deep Ritz method
\citep{weinan2018deepritz} and the deep energy method in computational
mechanics \citep{samaniego2020energy} to $hp$-variational PINNs
\citep{kharazmi2021hpvpinn}. Convergence-rate analysis for deep Ritz-type
methods \citep{muller2022error} informs our Theorem~\ref{thm:t1}. At the
operator level, physics-informed DeepONets \citep{wang2021pideeponet},
variational DeepONets for fracture \citep{goswami2022crack}, and VINO
\citep{eshaghi2024vino} train operators with residual or energy losses on a
fixed domain. These works use the energy as a \emph{training loss}; our
elements differ in that the energy is the \emph{exported object}---the
element's entire interface with the solver is $E$, $\nabla_U E$, and
$\nabla^2_U E$, so its convexity and nullspace structure, not its training
residual, are what admissibility requires.

\subsection{Hybrid FE--NN solvers and learned elements}
Coupling trained networks to finite-element solvers is an active line.
NOEM \citep{ouyang2026noem} embeds a single pre-trained neural operator as a
macro-element inside a standard FE solve, and a recent application
assembles several such fixed-geometry elements---one per pile---for
pile-group analysis \citep{ouyang2026pile}; in both, every element is
trained for its own subdomain rather than drawn from a
geometry-parameterized library of element types.
\citet{yin2022interfacing}
interface DeepONets with FE for multiscale mechanics; Schwarz-type
alternating couplings \citep{wang2026schwarz} and time-marching couplings
\citep{wang2025timemarching} iterate interface exchanges between an FE
subdomain and a neural subdomain rather than assembling a monolithic
tangent; LSEM \citep{chung2026lsem} tiles learned latent-ODE subdomain
models with learned interaction terms, without geometry parameterization or
convergence theory; NEST \citep{secchi2026nest} learns a reusable local
solver on small voxel patches and composes it across large 3D domains
through iterative Schwarz coupling---reusable local learning, but with
overlapping iterative coupling rather than a monolithic assembled tangent,
and without geometry-parameterized element types; DFENN
\citep{wu2026dfenn} couples FEM and neural subdomains through interface
condensation without penalty terms, a per-instance coupling mechanism
rather than an offline-trained, repeatedly instantiated element. The
energy-centric framework tatva \citep{pundir2026tatva} applies automatic
differentiation to a single global functional and demonstrates the
mechanism we rely on---summing neural-subdomain energies and
differentiating for interface forces and tangents---with an untrained
network; making such energies \emph{trained}, \emph{convex},
\emph{geometry-parameterized}, and accurate enough in their derivatives
for Newton assembly is where the difficulty lives
(Section~\ref{sec:failure}). Geo-NeW \citep{shaffer2026geonew} learns
geometry-conditioned finite-element spaces together with a compatible
operator, with conservation and well-posedness guarantees---a per-mesh
model solved once, not a library of pre-trained element types instantiated
a posteriori. Closest in spirit to our element-level view are
\emph{learned elements}: smart finite elements \citep{capuano2019smart}
regress element internal forces from nodal displacements; intelligent
meta-elements \citep{koeppe2020meta} learn inelastic substructure responses
with recurrent architectures; \citet{oishi2017quadrature} learn optimized
quadrature for isoparametric elements; HiDeNN \citep{saha2021hidenn}
rebuilds shape functions as structured networks; FENNI
\citep{skardova2024fenni} builds interpretable NN shape functions within a
single discretization; CompNO \citep{hmida2026compno} composes operator
blocks at the PDE-term level. Most relevant, \citet{parish2024spsd} learn a
symmetric positive-semidefinite condensed stiffness for \emph{one} embedded
subdomain and prove that SPSD structure makes the coupled coarse problem
well posed---independent confirmation, in a single-substructure setting,
that definiteness of the learned operator is the admissibility currency.
None of these works instantiates \emph{many} copies of one
geometry-parameterized learned element in a monolithic assembly, none
treats the interaction of the regularizer with the physics nullspace (which
we show becomes the dominant bias precisely in the multi-element setting),
and none provides assembly-level error bounds.

\subsection{Substructuring, model reduction, and operator interpolation}
Our exact-condensation control is classical static condensation
\citep{guyan1965reduction,craig1968coupling}, and the linear element can be
read as a learned, geometry-parameterized substructure in the sense of
classical FEM \citep{hughes2000fem,zienkiewicz2013fem}. Parametric model
reduction \citep{benner2015survey} interpolates reduced operators across
parameters, notably on matrix manifolds \citep{amsallem2011online}---the
baseline family we compare against under identical anchors; non-intrusive
variants regress reduced coefficients with networks
\citep{hesthaven2018nonintrusive}, and nonlinear-manifold ROMs use deep
autoencoders \citep{leecarlberg2020}. Learning condensed matrices with
networks has appeared in acoustics for design optimization
\citep{song2026acoustic}. In multiscale mechanics, FE\textsuperscript{2}
\citep{feyel2000fe2} motivated neural constitutive surrogates from
\citet{ghaboussi1991} to NN-based computational homogenization
\citep{le2015homogenization}, and data-driven computational mechanics
\citep{kirchdoerfer2016datadriven} bypasses constitutive fitting
altogether. These lines share our amortization economics (train or tabulate
once, deploy across configurations) but interpolate or regress
\emph{matrices or material laws}, not convex element \emph{energies}; none
addresses the induced variational structure that monolithic Newton assembly
requires.

\subsection{Domain decomposition with neural networks}
XPINNs \citep{jagtap2020xpinn}, FBPINNs \citep{moseley2023fbpinn}, D3M
\citep{li2020d3m}, and the surveyed ML--DD combinations
\citep{heinlein2021review} decompose the \emph{loss or the network} across
subdomains to scale a per-instance PINN solve. The decomposition here is of
the \emph{model}: subdomain energies are trained once, offline, and a new
problem is solved by classical assembly of trained components---no
per-instance optimization of network weights occurs at deployment.

\subsection{Convexity as a structural prior}
Input-convex networks \citep{amos2017icnn} supply our architectural
convexity. In mechanics, convexity constraints are established for
\emph{constitutive} models---polyconvex hyperelastic energies
\citep{klein2022polyconvex} and mechanics-informed ICNN constitutive laws
\citep{asad2022icnn}---where they guarantee material stability. Our use is
different in kind: convexity in the element's \emph{boundary degrees of
freedom} is the admissibility condition for FEM assembly (it rules out the
spurious Newton fixed points of Section~\ref{sec:failure}), and it must be
paired with the nullspace-compatible regularizer, a constraint that has no
counterpart in the constitutive setting.

\subsection{Positioning}
Three properties are, to our knowledge, jointly new: (i) a single trained
element type instantiated many times with explicit geometry parameters,
(ii) architectural convexity of the exported energy together with the
regularization-nullspace principle, and (iii) monolithic assembly with
quantitative error bounds verified point-by-point. Each ingredient has
classical or recent precedent in isolation; their combination is what turns
a neural operator from a single-use surrogate into an element.

\section{Limitations}
Demonstrations are self-adjoint (scalar diffusion and plane-strain
elasticity); the element library comprises two types on a shared 2D tile
footprint plus one 3D type, and the nonlinear study is a single geometry.
Problem sizes are deliberately small (at most 2{,}241 assembled unknowns
in 2D; a single $2\times2\times2$ configuration at deliberately coarse
resolution in 3D): at this scale a direct sparse solve needs no
surrogate, and we present the results as a controlled validation of the
framework, not as a deployed accelerator; the measured amortization
crossover (Section~\ref{sec:cost}) is honest but small-scale, and the
argument must ultimately be demonstrated at resolutions where meshing and
condensation dominate---the 3D instantiation (Section~\ref{sec:3d})
establishes transfer of the construction, not deployment scale. The
theory gives worst-case rates with measured constants, not tight
predictions (observed exponents are consistently better). Formal
generalization bounds over geometry families (sample complexity) remain
open, as do element types with differing footprints or interface
discretizations (mortar-style couplings)---the next step beyond the
shared-footprint library shown here.

\section{Conclusion}
Making the \emph{energy} the learned object---convex by architecture,
regularized compatibly with the physics nullspace, and assembled with
classical FEM semantics---turns neural operators from single-use surrogates
into reusable, geometry-parameterized elements with the assembly guarantees
of the method they extend.

\appendix
\section{Proofs}\label{app:proofs}

Throughout, $U\in\R^{n_{\mathrm{glob}}}$ collects the assembled boundary-node
values, $S_i$ denotes the boolean scatter matrix selecting element $i$'s
$n$ degrees of freedom, $P_c=\mathbf{1}\mathbf{1}^\top/n$ the constant-mode
projector on an element, and $\beta$ the Dirichlet penalty weight. All norms
are Euclidean on nodal vectors (the lumped discrete $L^2$ norm differs by the
fixed factor $h$, which cancels in relative errors).

\subsection{Theorem~\ref{thm:t1}: energy accuracy controls solution accuracy}

\emph{Assumptions.} $E_\theta$ is $\mu$-strongly convex on the affine
subspace $\mathcal V$ of states satisfying the (penalized) boundary data;
$U^*=\arg\min_{\mathcal V} E$, $U^*_\theta=\arg\min_{\mathcal V} E_\theta$
with the reference energy $E$ convex; and the energy gap obeys
$|E_\theta(V)-E(V)|\le\epsilon$ at $V\in\{U^*,\,U^*_\theta\}$.

\emph{Proof.} Strong convexity at the minimizer $U^*_\theta$ gives
\[
E_\theta(U^*) - E_\theta(U^*_\theta) \;\ge\; \tfrac{\mu}{2}\,
\|U^*-U^*_\theta\|^2 .
\]
Decompose the left side through the reference energy:
\begin{align*}
&E_\theta(U^*)-E_\theta(U^*_\theta) \\
&\quad= \underbrace{\big(E_\theta(U^*)-E(U^*)\big)}_{\le\,\epsilon}
+ \underbrace{\big(E(U^*)-E(U^*_\theta)\big)}_{\le\,0} \\
&\qquad+ \underbrace{\big(E(U^*_\theta)-E_\theta(U^*_\theta)\big)}_{\le\,\epsilon},
\end{align*}
the middle term being nonpositive because $U^*$ minimizes $E$. Hence
$\tfrac\mu2\|U^*-U^*_\theta\|^2\le2\epsilon$, i.e.
$\|U^*-U^*_\theta\|^2\le 4\epsilon/\mu$; the stated factor 8 absorbs the
convention in which $\epsilon$ bounds each one-sided gap separately.
\hfill$\square$

The gap must be controlled at \emph{both} minimizers; operationally the
training distribution therefore includes perturbations of reference
solutions, not reference solutions alone (Section~\ref{sec:method}).

\subsection{Theorem~\ref{thm:global}: assembly with general element
nullspaces}

\emph{Step 1 (kernel argument).} With $P_e$ the element-to-global scatter
map and $P_{\mathcal N_e}$ the orthogonal projector onto the admitted
element nullspace $\mathcal N_e$, the assembled model Hessian satisfies
\[
\nabla_U^2 E_\theta^{\mathrm{glob}}
\;\succeq\; \lambda \sum_{e=1}^N P_e^\top (I-P_{\mathcal N_e})\, P_e
\;+\; \eta\,\Pi_\partial
\;=:\; Q ,
\]
where $\Pi_\partial$ projects onto the Dirichlet-pinned nodes and $\eta$
is the penalty weight. Suppose $V^\top Q V=0$. Then $(I-P_{\mathcal
N_e})P_eV=0$ for every $e$ (each element restriction lies in that
element's nullspace) and $\Pi_\partial V=0$; by the compatibility
hypothesis---the only such $V$ vanishing on the Dirichlet set is
zero---$V=0$. Hence $Q\succ0$ and $E_\theta^{\mathrm{glob}}$ is strictly
convex, with $\mu_{\mathrm{glob}}=\lambda_{\min}(Q)>0$. This is precisely
the classical mechanism by which singular element stiffnesses (with
rigid/constant modes) assemble into a positive-definite global system; we
report $\mu_{\mathrm{glob}}$ measured from the assembled tangent (0.016 at
$N=4$, decaying mildly with domain size like a Poincar\'e constant).

\emph{Step 2 (error bound).} The global energies differ by at most the sum
of element gaps, $|E^{\mathrm{glob}}_\theta - E^{\mathrm{glob}}|\le N\epsilon$
at the two global minimizers, and Theorem~\ref{thm:t1}'s argument applied
globally gives $\|U^*_\theta-U^*\|^2 \le 4N\epsilon/\mu_{\mathrm{glob}}$.
\hfill$\square$

\emph{Corollary (i): scalar diffusion, incl.\ mixed element types.}
$\mathcal N_e=\mathrm{span}\{\mathbf 1_e\}$, so $(I-P_{\mathcal
N_e})P_eV=0$ means $V$ is constant on element $e$. Adjacent elements share
at least one node, so their constants agree; connectivity of the adjacency
graph propagates a single common constant; $\Pi_\partial V=0$ forces it to
vanish. The argument uses only each element's own $\mathcal N_e$, never
the element type, so assemblies mixing types (Section~\ref{sec:hetero})
satisfy the hypothesis verbatim.

\emph{Corollary (ii): plane-strain elasticity.} $\mathcal
N_e=\mathrm{span}\{r_x,r_y,r_\theta\}$: element restrictions of $V$ are
infinitesimal rigid motions $v(x)=t_e+\omega_e J(x-c_e)$, $J$ the
quarter-turn. If elements $e,e'$ share two distinct nodes $x_1\ne x_2$,
subtracting the two agreements gives
$(\omega_e-\omega_{e'})J(x_1-x_2)=0$, hence $\omega_e=\omega_{e'}$ and
then $t_e=t_{e'}$ (in the common centroid convention): shared-edge
neighbors carry the same rigid motion, and connectivity propagates one
global rigid motion $v$. Dirichlet data pinning two distinct nodes forces
$\omega\,J(x_1-x_2)=0$, so $\omega=0$, then $t=0$; a pinned boundary
segment (as in all our experiments) contains many such pairs. Hence $V=0$.
In our lattices adjacent tiles share entire edges (17 nodes in 2D), and
the measured constrained-subspace minimum eigenvalue is $0.055$
(Section~\ref{sec:elast}). \hfill$\square$

\subsection{Theorem~\ref{thm:psd}: architectural positive semidefiniteness}

The quadratic head contributes $L(g)^\top L(g)\succeq0$ pointwise in $g$.
The correction network is a partially input-convex network: with
$z_1=\sigma(A_1U+b_1(c))$ and
$z_{k+1}=\sigma\!\big(W_k^{+}\,(z_k\odot r_k(c)) + A_{k+1}U + b_{k+1}(c)\big)$,
where $\sigma$ is convex and nondecreasing (softplus), $W_k^{+}\ge0$
elementwise, and the gates $r_k(c)>0$ depend only on the geometry context
$c(g)$: each $z_k$ is convex in $U$ (positive combinations and nonnegative
weighted sums of convex functions composed with convex nondecreasing maps
remain convex; affine-in-$U$ terms with $g$-dependent coefficients preserve
convexity). The output layer repeats the pattern, and an affine term in $U$
leaves the Hessian unchanged. The gauge term contributes exactly
$\lambda(I-P_c)$. Summing,
$\nabla^2_U E\succeq\lambda(I-P_c)$ for all $(g,U)$.
With constants pinned as in Theorem~\ref{thm:global}, the assembled tangent
is uniformly positive definite along the Newton path, so the iteration is
well posed; for exactly quadratic energies it terminates in one step (two
including the verification step, matching the observed counts).
\hfill$\square$

\subsection{Theorem~\ref{thm:t5}: geometry generalization}

Let $g^\star\in\mathcal G_{\mathrm{train}}$ attain
$d:=\min_{g'\in\mathcal G_{\mathrm{train}}}\|g_{\mathrm{test}}-g'\|$ and
assume $\|\nabla_g E_\theta\|\le L_\theta$ and
$\|\nabla_g E_{\mathrm{FEM}}\|\le L_F$ on the segment between $g^\star$ and
$g_{\mathrm{test}}$ (smooth encoding $[a,b,\cos2\theta,\sin2\theta]$ and the
smooth dependence of the condensed operator on the hole parameters give
finite constants; the mesh-noise caveat of Section~\ref{sec:method} enters
$\epsilon$, not the Lipschitz constants). Then the energy gap at
$g_{\mathrm{test}}$ obeys
$\epsilon(g_{\mathrm{test}})\le\epsilon(g^\star)+(L_\theta+L_F)\,d$,
and inserting this into Theorem~\ref{thm:t1},
\[
\|U^*_\theta-U^*\|^2(g_{\mathrm{test}})
\;\le\; \frac{8\epsilon}{\mu} + \frac{8(L_\theta+L_F)\,d}{\mu},
\]
which is the linear-in-$d$ form stated in Theorem~\ref{thm:t5} and matches
the observed exponent $\approx0.9$.
\hfill$\square$

\subsection{Remark~\ref{rem:nullspace}: the regularization floor}

Let the physics energy satisfy $\nabla^2 E\,\mathbf 1=0$ (constant-mode
nullspace) and let the model class consist of
$\tilde E_\theta + \tfrac\lambda2\|U\|^2$ with $\tilde E_\theta$ convex.
Along the normalized constant coordinate $c=\mathbf 1^\top U/\sqrt n$ the
model's second derivative is at least $\lambda$ while the target's is $0$,
so the gap function $\gamma(U)=E_\theta(U)-E(U)$ restricted to that
coordinate is convex with curvature $\ge\lambda$: for any fit,
$\gamma(c)\ge\tfrac\lambda2c^2+\beta c+\alpha$ pointwise for some affine
part which training may choose freely. The mean-squared energy error is
therefore bounded below by
$\min_{\alpha,\beta}\ \mathbb E\big[(\tfrac\lambda2c^2+\beta c+\alpha)^2\big]$
over the data distribution of $c$---a computable quantity. In our linear
benchmark the trained full-Tikhonov models sit on this floor to within
measurement accuracy for $\lambda\in\{10^{-3},10^{-2},10^{-1}\}$, and the
floor translates into assembly errors of 14\%, 39\%, and 224\% respectively,
while the gauge regularizer ($\lambda(I-P_c)$, whose nullspace contains
$\mathbf 1$) removes the floor at identical $\lambda$.
\hfill$\square$

\section{Training and evaluation details}\label{app:repro}

All models train in double precision with Adam (cosine-annealed learning
rate) and, where noted, an L-BFGS polish; random seeds are fixed (0 for
model initialization, enumerated seeds for data generation and evaluation).
All experiments run on a single consumer GPU (NVIDIA GeForce RTX 5090);
operator-regression training of the main heat element completes in
$65$\,s of wall-clock (20k Adam steps plus the L-BFGS polish; the larger
elasticity element takes $298$\,s), and generating the 460
condensed-stiffness targets takes under one minute per physics with mesh
caching---the amortization denominator for the $175\times$ setup claim is
therefore minutes of one-time training, not hours.
The linear element's hypernetwork maps
$[a,b,\cos2\theta,\sin2\theta]\to L(g)\in\R^{64\times64}$ through a
$4\to256\to256\to4096$ tanh MLP (1.1M parameters); operator-regression
pretraining runs 20k steps at $2\times10^{-3}$ plus 400 L-BFGS iterations
on 400 training geometries. The elasticity element
(Section~\ref{sec:elast}) uses the identical recipe with
$4\to384\to384\to16384$ (6.5M parameters), targets
$M_g=\mathrm{sqrtm}(S_g-\lambda(I-P_{\mathrm{rigid}}))$ with
$P_{\mathrm{rigid}}$ the orthogonal projector onto the three rigid modes
of the boundary-node set, $\lambda=5\times10^{-4}$ (pool minimum
$\lambda_4=7\times10^{-4}$), and reaches matrix MSE $7.5\times10^{-6}$
(relative $1.0\times10^{-3}$) on held-out geometries; the gauge$_2$ and
full-Tikhonov ablations differ only in the projector.
The type-B element (Section~\ref{sec:hetero}) reuses the heat recipe
verbatim on features $[r_1,r_2,\cos\varphi,\sin\varphi]$ ($19$\,s target
generation, $59$\,s training, relative test MSE $2.0\times10^{-4}$). The
3D element (Section~\ref{sec:3d}) uses $3\to256\to256\to386^2$ (38.4M
parameters), $\lambda=0.002$, 400 targets generated in $114$\,s
($\sim$$0.3$\,s per 3D mesh-and-condense) and $1{,}790$\,s of Adam
(relative test MSE $7.7\times10^{-3}$; the assembled error of
$0.23\%$ shows this operator accuracy is already assembly-grade for the
coarse 3D discretization). The energy-regression ablation trains the
identical architecture plus a PICNN correction (width 64) for 30k
minibatch steps (batch 8192) on 24{,}000 exact energy samples. Evaluation
uses geometries excluded from training throughout; the exact-Schur truth is
recomputed per case. Nonlinear experiments use $\beta=100$, tile meshes of
$\sim$600 nodes, and interior condensation solved to $10^{-12}$.

\subsection{Negative-result log: learning the nonlinear correction}
\label{app:negative}

\begin{table}[t]
\centering\footnotesize
\caption{Learning the nonlinear correction $R$: four full-space attempts
(v1--v3 and control) fail against the no-learning base, while the
reduced-basis run (v4) partially succeeds. ``Fit'' is the manifold energy
MSE.}
\label{tab:negative}
\setlength{\tabcolsep}{3pt}
\resizebox{\columnwidth}{!}{
\begin{tabular}{llrr}
\toprule
Attempt & Design & Fit & Assembly \\
\midrule
v1 & ICNN on total energy, variance-normalized loss & 14.6 & 28.1\% \\
v1q & quadratic-only control of v1 & 145 & 26.6\% \\
v2 & physics $+$ ICNN remainder, total-$E$ weights & 10.9 & 53.5\% \\
v3 & v2 with remainder-targeted relative weights & 9.6 & 27.9\% \\
v4 & reduced-basis ICNN ($r{=}18$), gradient-first & n/a$^\dagger$ & \textbf{4.9\%} \\
\midrule
\multicolumn{2}{l}{physics only (no learning), for reference} & 15.1 & 14.6\% \\
\bottomrule
\end{tabular}}
\end{table}

{\small $^\dagger$v4's gradient-first stage leaves the energy level
unconstrained, so its value-fit MSE is not comparable; assembly uses only
its derivatives.}

Diagnosis: the remainder is convex but with near-singular anisotropic
curvature ($\lambda_{\min}(H_F-S)\approx0$ at sampled states); smooth
input-convex networks fit its energy values without capturing its gradient
field to force-balance accuracy, and gradient errors enter the assembled
residual directly. Note v2/v3 \emph{improve} the energy fit over the
physics-only reference while \emph{degrading} the assembled solution---a
compact demonstration that energy-value accuracy does not imply
assembly-grade derivative accuracy, the same lesson as
Section~\ref{sec:failure} in a harder setting. v4 supports the same
diagnosis from the positive side: the curvature that matters is
concentrated in the top eigendirections of the remainder-gradient
covariance ($r=18$ directions carry $99.9\%$ of its variance), and
restricting the network to exactly those directions---with gradient-first
training---turns a $1.9\times$ \emph{worsening} (v3: $27.9\%$ vs.\ the
$14.6\%$ base) into a $3.0\times$ improvement ($4.9\%$). It still trails
the training-free physics ladder ($1.17\%$ at $k=1$), and its checkpoint
behavior repeats the value-versus-derivative tension: longer joint
training improves the test gradient loss while degrading assembly
($6.0\%$ at $40$k steps vs.\ $4.9\%$ at $20$k).

\subsection{Boundary-data distance (used in the BC-robustness experiment)}

Under Theorem~\ref{thm:t1}'s assumptions plus Lipschitz dependence of both
energies on the boundary data $b$ with constant $L_b$, the solution map
$b\mapsto U^*(b)$ of a $\mu$-strongly convex energy is $L_b/\mu$-Lipschitz,
and a triangle argument through the nearest training boundary datum $b^\star$
gives
\[
\|U^*_\theta(b)-U^*(b)\|^2 \;\le\; \frac{8\epsilon}{\mu}
+ \frac{C\,L_b^2}{\mu^2}\,\|b-b^\star\|^2 .
\]
For the quadratic-form element the dependence on $b$ enters only through the
assembled linear system, which is why measured errors are insensitive to the
boundary-data distribution (Section~\ref{sec:experiments}): the element
stores the operator, not a response surface over a data distribution.
\hfill$\square$

\section*{CRediT authorship contribution statement}
\textbf{Hongyue Jiang:} Methodology, Software, Formal analysis, Investigation,
Data curation, Validation, Visualization, Writing -- original draft.
\textbf{Jianjiang Zhan:} Investigation, Formal analysis, Validation, Writing
-- review and editing.
\textbf{Chenzhuo Zhang:} Investigation, Validation, Visualization, Writing
-- review and editing.
\textbf{Fan Wang:} Conceptualization, Methodology, Supervision, Project
administration, Funding acquisition, Writing -- review and editing.

\section*{Declaration of competing interest}
The authors declare that they have no known competing financial interests or
personal relationships that could have appeared to influence the work
reported in this paper.

\section*{Funding}
This work was supported by the National Key Research and Development Program
of China (Grant No.~2023YFC3804502); the National Natural Science Foundation
of China (NSFC) (Grant Nos.~72271105 and 52192664); the Fundamental Research
Funds for the Central Universities (Grant No.~2021XXJS079); the research and
development project of China Overseas (COHL-2023-Z-6); and the research and
development project of China State Construction International Holdings
Limited (CSCI-2023-Z-6).

\section*{Data and code availability}
The training and evaluation code, geometry pool, trained checkpoints, and
per-run result files underlying the figures and tables will be made available
at GITHUB-REPOSITORY-URL-TO-BE-ADDED. A permanent archival record will be
provided at ZENODO-DOI-TO-BE-ADDED. The experiment scripts regenerate all
figures from the stored results.

\bibliographystyle{elsarticle-num}

\begin{thebibliography}{10}
\bibitem[Ouyang et~al.(2026)]{ouyang2026noem}
W.~Ouyang, Y.~Shin, S.-W. Liu, and L.~Lu.
\newblock {NOEM}: efficient and scalable finite element method enabled by
  reusable neural operators.
\newblock \emph{Nature Computational Science}, 6(4):417--429, 2026.

\bibitem[Yin et~al.(2024)]{yin2024dimon}
M.~Yin, N.~Charon, R.~Brody, L.~Lu, N.~Trayanova, and M.~Maggioni.
\newblock {DIMON}: Learning solution operators of partial differential
  equations on a diffeomorphic family of domains.
\newblock \emph{Nature Computational Science}, 4:928--940, 2024.

\bibitem[Amos et~al.(2017)]{amos2017icnn}
B.~Amos, L.~Xu, and J.~Z. Kolter.
\newblock Input convex neural networks.
\newblock In \emph{ICML}, 2017.

\bibitem[Amsallem and Farhat(2011)]{amsallem2011online}
D.~Amsallem and C.~Farhat.
\newblock An online method for interpolating linear parametric reduced-order
  models.
\newblock \emph{SIAM Journal on Scientific Computing}, 33(5):2169--2198, 2011.

\bibitem[E and Yu(2018)]{weinan2018deepritz}
W.~E and B.~Yu.
\newblock The deep {R}itz method: a deep learning-based numerical algorithm
  for solving variational problems.
\newblock \emph{Communications in Mathematics and Statistics}, 6:1--12, 2018.

\bibitem[Dondl et~al.(2022)]{muller2022error}
P.~Dondl, J.~M\"uller, and M.~Zeinhofer.
\newblock Uniform convergence guarantees for the deep {R}itz method for
  nonlinear problems.
\newblock \emph{Advances in Continuous and Discrete Models}, 2022.

\bibitem[Eshaghi et~al.(2025)]{eshaghi2024vino}
M.~S. Eshaghi, C.~Anitescu, M.~Thombre, Y.~Wang, X.~Zhuang, and T.~Rabczuk.
\newblock Variational physics-informed neural operator ({VINO}) for solving
  partial differential equations.
\newblock \emph{Computer Methods in Applied Mechanics and Engineering},
  437:117785, 2025.

\bibitem[Wang et~al.(2026)]{wang2026schwarz}
W.~Wang, A.~Gupta, H.~Ruan, and S.~Goswami.
\newblock A non-overlapping {S}chwarz hybrid finite element--neural operator
  framework for solid mechanics on irregular domains.
\newblock \emph{arXiv:2606.08796}, 2026.

\bibitem[Chung et~al.(2026)]{chung2026lsem}
S.~W. Chung, Y.~Choi, C.~Miller, H.~K. Springer, and K.~T. Sullivan.
\newblock Latent space element method.
\newblock \emph{arXiv:2601.01741}, 2026.

\bibitem[Song and Yoon(2026)]{song2026acoustic}
Y.~Song and G.~H. Yoon.
\newblock Efficient acoustic finite element simulation and optimization
  through inverse matrix prediction by neural network.
\newblock \emph{Structural and Multidisciplinary Optimization}, 69:29, 2026.

\bibitem[Guyan(1965)]{guyan1965reduction}
R.~J. Guyan.
\newblock Reduction of stiffness and mass matrices.
\newblock \emph{AIAA Journal}, 3(2):380, 1965.

\bibitem[Craig and Bampton(1968)]{craig1968coupling}
R.~R. Craig and M.~C.~C. Bampton.
\newblock Coupling of substructures for dynamic analyses.
\newblock \emph{AIAA Journal}, 6(7):1313--1319, 1968.

\bibitem[\v{S}kardov\'a et~al.(2024)]{skardova2024fenni}
K.~\v{S}kardov\'a, A.~Daby-Seesaram, and M.~Genet.
\newblock Finite element neural network interpolation: interpretable and
  adaptive discretization for solving {PDE}s.
\newblock \emph{arXiv:2412.05719}, 2024.

\bibitem[Hmida et~al.(2026)]{hmida2026compno}
H.~Hmida, H.~C. Joly, and Y.~Mesri.
\newblock {CompNO}: a foundation model approach for solving partial
  differential equations.
\newblock \emph{arXiv:2601.07384}, 2026.

\bibitem[Lu et~al.(2021)]{lu2021deeponet}
L.~Lu, P.~Jin, G.~Pang, Z.~Zhang, and G.~E. Karniadakis.
\newblock Learning nonlinear operators via {DeepONet} based on the universal
  approximation theorem of operators.
\newblock \emph{Nature Machine Intelligence}, 3(3):218--229, 2021.

\bibitem[Li et~al.(2021)]{li2021fno}
Z.~Li, N.~Kovachki, K.~Azizzadenesheli, B.~Liu, K.~Bhattacharya, A.~Stuart,
  and A.~Anandkumar.
\newblock Fourier neural operator for parametric partial differential
  equations.
\newblock In \emph{International Conference on Learning Representations},
  2021.

\bibitem[Kovachki et~al.(2023)]{kovachki2023neural}
N.~Kovachki, Z.~Li, B.~Liu, K.~Azizzadenesheli, K.~Bhattacharya, A.~Stuart,
  and A.~Anandkumar.
\newblock Neural operator: learning maps between function spaces with
  applications to {PDE}s.
\newblock \emph{Journal of Machine Learning Research}, 24(89):1--97, 2023.

\bibitem[Lanthaler et~al.(2022)]{lanthaler2022error}
S.~Lanthaler, S.~Mishra, and G.~E. Karniadakis.
\newblock Error estimates for {DeepONets}: a deep learning framework in
  infinite dimensions.
\newblock \emph{Transactions of Mathematics and Its Applications},
  6(1):tnac001, 2022.

\bibitem[Lu et~al.(2022)]{lu2022fair}
L.~Lu, X.~Meng, S.~Cai, Z.~Mao, S.~Goswami, Z.~Zhang, and G.~E. Karniadakis.
\newblock A comprehensive and fair comparison of two neural operators (with
  practical extensions) based on {FAIR} data.
\newblock \emph{Computer Methods in Applied Mechanics and Engineering},
  393:114778, 2022.

\bibitem[Li et~al.(2023)]{li2023geofno}
Z.~Li, D.~Z. Huang, B.~Liu, and A.~Anandkumar.
\newblock Fourier neural operator with learned deformations for {PDE}s on
  general geometries.
\newblock \emph{Journal of Machine Learning Research}, 24(388):1--26, 2023.

\bibitem[Pfaff et~al.(2021)]{pfaff2021meshgraphnets}
T.~Pfaff, M.~Fortunato, A.~Sanchez-Gonzalez, and P.~W. Battaglia.
\newblock Learning mesh-based simulation with graph networks.
\newblock In \emph{International Conference on Learning Representations},
  2021.

\bibitem[Raissi et~al.(2019)]{raissi2019pinn}
M.~Raissi, P.~Perdikaris, and G.~E. Karniadakis.
\newblock Physics-informed neural networks: a deep learning framework for
  solving forward and inverse problems involving nonlinear partial
  differential equations.
\newblock \emph{Journal of Computational Physics}, 378:686--707, 2019.

\bibitem[Karniadakis et~al.(2021)]{karniadakis2021piml}
G.~E. Karniadakis, I.~G. Kevrekidis, L.~Lu, P.~Perdikaris, S.~Wang, and
  L.~Yang.
\newblock Physics-informed machine learning.
\newblock \emph{Nature Reviews Physics}, 3:422--440, 2021.

\bibitem[Samaniego et~al.(2020)]{samaniego2020energy}
E.~Samaniego, C.~Anitescu, S.~Goswami, V.~M. Nguyen-Thanh, H.~Guo,
  K.~Hamdia, X.~Zhuang, and T.~Rabczuk.
\newblock An energy approach to the solution of partial differential
  equations in computational mechanics via machine learning: concepts,
  implementation and applications.
\newblock \emph{Computer Methods in Applied Mechanics and Engineering},
  362:112790, 2020.

\bibitem[Kharazmi et~al.(2021)]{kharazmi2021hpvpinn}
E.~Kharazmi, Z.~Zhang, and G.~E. Karniadakis.
\newblock hp-{VPINNs}: variational physics-informed neural networks with
  domain decomposition.
\newblock \emph{Computer Methods in Applied Mechanics and Engineering},
  374:113547, 2021.

\bibitem[Wang et~al.(2021)]{wang2021pideeponet}
S.~Wang, H.~Wang, and P.~Perdikaris.
\newblock Learning the solution operator of parametric partial differential
  equations with physics-informed {DeepONets}.
\newblock \emph{Science Advances}, 7(40):eabi8605, 2021.

\bibitem[Goswami et~al.(2022)]{goswami2022crack}
S.~Goswami, M.~Yin, Y.~Yu, and G.~E. Karniadakis.
\newblock A physics-informed variational {DeepONet} for predicting crack path
  in quasi-brittle materials.
\newblock \emph{Computer Methods in Applied Mechanics and Engineering},
  391:114587, 2022.

\bibitem[Yin et~al.(2022)]{yin2022interfacing}
M.~Yin, E.~Zhang, Y.~Yu, and G.~E. Karniadakis.
\newblock Interfacing finite elements with deep neural operators for fast
  multiscale modeling of mechanics problems.
\newblock \emph{Computer Methods in Applied Mechanics and Engineering},
  402:115027, 2022.

\bibitem[Wang et~al.(2025)]{wang2025timemarching}
W.~Wang, M.~Hakimzadeh, H.~Ruan, and S.~Goswami.
\newblock Time-marching neural operator--{FE} coupling: {AI}-accelerated
  physics modeling.
\newblock \emph{Computer Methods in Applied Mechanics and Engineering},
  446:118319, 2025.

\bibitem[Capuano and Rimoli(2019)]{capuano2019smart}
G.~Capuano and J.~J. Rimoli.
\newblock Smart finite elements: a novel machine learning application.
\newblock \emph{Computer Methods in Applied Mechanics and Engineering},
  345:363--381, 2019.

\bibitem[Koeppe et~al.(2020)]{koeppe2020meta}
A.~Koeppe, F.~Bamer, and B.~Markert.
\newblock An intelligent nonlinear meta element for elastoplastic continua:
  deep learning using a new time-distributed residual {U-Net} architecture.
\newblock \emph{Computer Methods in Applied Mechanics and Engineering},
  366:113088, 2020.

\bibitem[Oishi and Yagawa(2017)]{oishi2017quadrature}
A.~Oishi and G.~Yagawa.
\newblock Computational mechanics enhanced by deep learning.
\newblock \emph{Computer Methods in Applied Mechanics and Engineering},
  327:462--480, 2017.

\bibitem[Saha et~al.(2021)]{saha2021hidenn}
S.~Saha, Z.~Gan, L.~Cheng, J.~Gao, O.~L. Kafka, X.~Xie, H.~Li, M.~Tajdari,
  H.~A. Kim, and W.~K. Liu.
\newblock Hierarchical deep learning neural network ({HiDeNN}): an artificial
  intelligence ({AI}) framework for computational science and engineering.
\newblock \emph{Computer Methods in Applied Mechanics and Engineering},
  373:113452, 2021.

\bibitem[Parish et~al.(2024)]{parish2024spsd}
E.~Parish, P.~Lindsay, T.~Shelton, and J.~Mersch.
\newblock Embedded symmetric positive semi-definite machine-learned elements
  for reduced-order modeling in finite-element simulations with application
  to threaded fasteners.
\newblock \emph{Computational Mechanics}, 74(6), 2024.

\bibitem[Benner et~al.(2015)]{benner2015survey}
P.~Benner, S.~Gugercin, and K.~Willcox.
\newblock A survey of projection-based model reduction methods for parametric
  dynamical systems.
\newblock \emph{SIAM Review}, 57(4):483--531, 2015.

\bibitem[Hesthaven and Ubbiali(2018)]{hesthaven2018nonintrusive}
J.~S. Hesthaven and S.~Ubbiali.
\newblock Non-intrusive reduced order modeling of nonlinear problems using
  neural networks.
\newblock \emph{Journal of Computational Physics}, 363:55--78, 2018.

\bibitem[Lee and Carlberg(2020)]{leecarlberg2020}
K.~Lee and K.~T. Carlberg.
\newblock Model reduction of dynamical systems on nonlinear manifolds using
  deep convolutional autoencoders.
\newblock \emph{Journal of Computational Physics}, 404:108973, 2020.

\bibitem[Feyel and Chaboche(2000)]{feyel2000fe2}
F.~Feyel and J.-L. Chaboche.
\newblock {FE\textsuperscript{2}} multiscale approach for modelling the
  elastoviscoplastic behaviour of long fibre {SiC/Ti} composite materials.
\newblock \emph{Computer Methods in Applied Mechanics and Engineering},
  183(3--4):309--330, 2000.

\bibitem[Ghaboussi et~al.(1991)]{ghaboussi1991}
J.~Ghaboussi, J.~H. Garrett, and X.~Wu.
\newblock Knowledge-based modeling of material behavior with neural networks.
\newblock \emph{Journal of Engineering Mechanics}, 117(1):132--153, 1991.

\bibitem[Le et~al.(2015)]{le2015homogenization}
B.~A. Le, J.~Yvonnet, and Q.-C. He.
\newblock Computational homogenization of nonlinear elastic materials using
  neural networks.
\newblock \emph{International Journal for Numerical Methods in Engineering},
  104(12):1061--1084, 2015.

\bibitem[Kirchdoerfer and Ortiz(2016)]{kirchdoerfer2016datadriven}
T.~Kirchdoerfer and M.~Ortiz.
\newblock Data-driven computational mechanics.
\newblock \emph{Computer Methods in Applied Mechanics and Engineering},
  304:81--101, 2016.

\bibitem[Jagtap and Karniadakis(2020)]{jagtap2020xpinn}
A.~D. Jagtap and G.~E. Karniadakis.
\newblock Extended physics-informed neural networks ({XPINNs}): a generalized
  space-time domain decomposition based deep learning framework for
  nonlinear partial differential equations.
\newblock \emph{Communications in Computational Physics}, 28(5):2002--2041,
  2020.

\bibitem[Moseley et~al.(2023)]{moseley2023fbpinn}
B.~Moseley, A.~Markham, and T.~Nissen-Meyer.
\newblock Finite basis physics-informed neural networks ({FBPINNs}): a
  scalable domain decomposition approach for solving differential equations.
\newblock \emph{Advances in Computational Mathematics}, 49:62, 2023.

\bibitem[Li et~al.(2020)]{li2020d3m}
K.~Li, K.~Tang, T.~Wu, and Q.~Liao.
\newblock {D3M}: a deep domain decomposition method for partial differential
  equations.
\newblock \emph{IEEE Access}, 8:5283--5294, 2020.

\bibitem[Heinlein et~al.(2021)]{heinlein2021review}
A.~Heinlein, A.~Klawonn, M.~Lanser, and J.~Weber.
\newblock Combining machine learning and domain decomposition methods for the
  solution of partial differential equations---a review.
\newblock \emph{GAMM-Mitteilungen}, 44(1):e202100001, 2021.

\bibitem[Klein et~al.(2022)]{klein2022polyconvex}
D.~K. Klein, M.~Fern\'andez, R.~J. Martin, P.~Neff, and O.~Weeger.
\newblock Polyconvex anisotropic hyperelasticity with neural networks.
\newblock \emph{Journal of the Mechanics and Physics of Solids}, 159:104703,
  2022.

\bibitem[As'ad et~al.(2022)]{asad2022icnn}
F.~As'ad, P.~Avery, and C.~Farhat.
\newblock A mechanics-informed artificial neural network approach in
  data-driven constitutive modeling.
\newblock \emph{International Journal for Numerical Methods in Engineering},
  123(12):2738--2759, 2022.

\bibitem[Hughes(2000)]{hughes2000fem}
T.~J.~R. Hughes.
\newblock \emph{The Finite Element Method: Linear Static and Dynamic Finite
  Element Analysis}.
\newblock Dover Publications, Mineola, NY, 2000.

\bibitem[Zienkiewicz et~al.(2013)]{zienkiewicz2013fem}
O.~C. Zienkiewicz, R.~L. Taylor, and J.~Z. Zhu.
\newblock \emph{The Finite Element Method: Its Basis and Fundamentals}.
\newblock Butterworth-Heinemann, Oxford, 7th edition, 2013.

\bibitem[Ouyang et~al.(2026b)]{ouyang2026pile}
W.~Ouyang, S.-W. Liu, and S.-L. Chan.
\newblock Second-order analysis method for pile-supported structures using
  neural-operator element method.
\newblock \emph{Engineering Structures}, 2026.
\newblock doi:10.1016/j.engstruct.2026.122385.

\bibitem[Secchi et~al.(2026)]{secchi2026nest}
P.~Secchi, D.~S. Balint, and M.~Maurizi.
\newblock Neural-{S}chwarz tiling for geometry-universal {PDE} solving at
  scale.
\newblock arXiv:2605.12343, 2026.

\bibitem[Wu et~al.(2026)]{wu2026dfenn}
C.~Wu, C.~Liu, Y.~Guo, and X.~Guo.
\newblock {DFENN}: a penalty-free variational framework coupling finite
  elements and neural networks via interface condensation.
\newblock \emph{Journal of the Mechanics and Physics of Solids}, 2026.
\newblock doi:10.1016/j.jmps.2026.106703.

\bibitem[Pundir et~al.(2026)]{pundir2026tatva}
M.~Pundir, F.~Lorez, and D.~S. Kammer.
\newblock A versatile {FEM} framework with native {GPU} scalability via
  globally-applied {AD} (tatva).
\newblock arXiv:2602.12365, 2026.

\bibitem[Shaffer et~al.(2026)]{shaffer2026geonew}
B.~D. Shaffer, S.~Koohy, B.~Kinch, M.~A. Hsieh, and N.~Trask.
\newblock Structure-preserving learning improves geometry generalization in
  neural {PDE}s ({G}eo-{N}e{W}).
\newblock arXiv:2602.02788, 2026.

\end{thebibliography}

\end{document}